\documentclass{article}

\usepackage{microtype}
\usepackage{graphicx}
\usepackage{subfigure}
\usepackage{booktabs} % for professional tables

\usepackage{hyperref}

\usepackage{amsmath}
\usepackage{amssymb}
\usepackage{mathtools}
\usepackage{amsthm}
\usepackage{bm}

\usepackage[capitalize,noabbrev]{cleveref}

\theoremstyle{plain}
\newtheorem{theorem}{Theorem}[section]
\newtheorem{proposition}[theorem]{Proposition}

\theoremstyle{definition}

\theoremstyle{remark}

\renewcommand{\Pr}{\operatorname*{\mathbb{P}}}

\newcommand{\argmax}{\operatorname*{\mathrm{arg\,max}}}
\newcommand{\argmin}{\operatorname*{\mathrm{arg\,min}}}

\newcommand{\D}{\mathcal{D}}

\newcommand{\eps}{\varepsilon}
\newcommand{\bpsi}{\bm{\psi}}

\newcommand{\domain}{\mathcal{X}}
\newcommand{\fatt}[1]{\psi_{#1}}
\newcommand{\mat}[2]{A_{#1#2}}
\newcommand{\attmat}{\bm{\mat{}{}}}

\usepackage[show]{chato-notes}
\usepackage{paralist}
\usepackage{wrapfig}
\usepackage[preprint]{neurips_2022}

\usepackage[utf8]{inputenc} % allow utf-8 input
\usepackage[T1]{fontenc}    % use 8-bit T1 fonts
\usepackage{hyperref}       % hyperlinks
\usepackage{url}            % simple URL typesetting
\usepackage{booktabs}       % professional-quality tables
\usepackage{amsfonts}       % blackboard math symbols
\usepackage{nicefrac}       % compact symbols for 1/2, etc.
\usepackage{microtype}      % microtypography
\usepackage{xcolor}         % colors

\title{Tight Lower Bounds on Worst-Case Guarantees for Zero-Shot Learning with Attributes}
\author{
  Alessio Mazzetto$^{*}$
    \\
  Brown University \\
   \And
   Cristina Menghini$^{*}$ \\
   Brown University
   \AND
   Andrew Yuan \\
   Brown University 
   \And
   Eli Upfal \\
   Brown University 
   \And
   Stephen H. Bach \\
   Brown University
}

\newcommand\nnfootnote[1]{%
  \begin{NoHyper}
  \renewcommand\thefootnote{}\footnote{#1}%
  \addtocounter{footnote}{-1}%
  \end{NoHyper}
}

\begin{document}

\maketitle
\nnfootnote{$^*$ Equal contribution.}

\begin{abstract}
We develop a rigorous mathematical analysis of zero-shot learning with attributes.
In this setting, the goal is to label novel classes with no training data, only detectors for attributes and a description of how those attributes are correlated with the target classes, called the class-attribute matrix.
We develop the first non-trivial lower bound on the worst-case error of the best map from attributes to classes for this setting, even with perfect attribute detectors.
The lower bound characterizes the theoretical intrinsic difficulty of the zero-shot problem based on the available information---the class-attribute matrix---and the bound is practically computable from it.
Our lower bound is tight, as we show that we can always find a randomized map from attributes to classes whose expected error is upper bounded by the value of the lower bound.
We show that our analysis can be predictive of how standard zero-shot methods behave in practice, including which classes will likely be confused with others.
%\eli{The first and only method to do that??}
\end{abstract}

\section{Introduction}
\label{introduction}
Labeled training data is often scarce or unavailable, and it can be very costly to obtain.
For this reason, there is a growing interest in developing methods that can exploit source of information other than labeled data, such as zero-shot learning (ZSL).
In ZSL, we want to recognize items of \textit{unseen classes}, for which labeled data is not available.
A ZSL model is trained on a disjoint set of similar classes, called \textit{seen classes}, for which labeled data is available instead.
The model is trained to map examples to auxiliary information describing the seen classes.
Then, at test time, predictions can be made using only descriptions of the unseen classes.
While ZSL is increasingly common in practice, from a theoretical perspective ZSL is a hard problem that defies analysis, because in the worst case there can be an arbitrary shift between the distributions of the seen and unseen classes.
In this work, we take a step towards a better theoretical understanding of ZSL.
We investigate the question: \emph{Given only auxiliary information in the form of attributes describing unseen classes, what is the smallest worst-case error than {\bf any} method can guarantee?}
We provide the first non-trivial answer to this question by developing a framework based on adversarial optimization.
We also show that this framework has practical application as a method for identifying when the predictions of ZSL methods on certain unseen classes are more likely to be incorrect.

ZSL models have obtained impressive accuracy in practice, both for vision~\citep{xian2018zero} and language domains~\citep{sanh:iclr22,wei:iclr22}, but they come with no theoretical characterization of their accuracy.
To address this gap, we analyze the attribute-based ZSL setting that includes a large portion of the classic methods proposed in the literature \citep{RomeraParedes2015AnES,Lampert2014AttributeBasedCF,Akata2015cvpr,Akata2016}, as well as more recent end-to-end deep learning approaches~\citep{Kodirov2017cvpr,xian2018zero,Huynh-DAZLE:CVPR20}.
While this setting does not include all varieties of ZSL (discussed further in Section~\ref{sec:related-work}), we view this work as a critical first step towards building up a broader theory of ZSL. In attribute-based ZSL, an attribute is a property of a item to be classified.
Each item can either exhibit a given attribute or not.
For example, an image of a lion would often exhibit the attribute tail, while the image of a sheep would not. 
Attribute-based ZSL models are trained using \textit{attribute representations} of the items of the seen classes, and a \textit{class-attribute} matrix that describes pairwise relations between the seen classes and each attribute.
At test time, predictions are made for the unseen classes given the items' attribute representation and a new class-attribute matrix.

\citet{RomeraParedes2015AnES} is one of the few works to address theoretical questions related to ZSL.
Studying attribute-based ZSL, they show a pair of basic bounds that characterize sufficient conditions for either learning or impossibility:
(1) if there is no shift from the seen to the unseen classes, then learning is trivial, and (2) if the vectors of attributes of the seen classes are mutually orthogonal with those of the unseen classes, then the error can be arbitrarily large in the worst case.
In this paper, we provide the first non-trivial lower bound for ZSL with attributes, addressing the open problem posed by~\citet{RomeraParedes2015AnES}.

%\eli{The first two sentences repeat what was said in early paragraph - the one I suggest to move here.}
%\steve{Waiting to see what we decide about needing the background in paragraph 2. If we keep it there, I can trim this to just refer to the two stages already mentioned.}
We analyze ZSL with attributes by first observing that it is a two stage process consisting of a training phase and an inference phase.
In the training phase, we learn a map from the items to the attribute space using the seen classes, while in the inference phase we use the class-attribute matrix to infer the correct class given the item-attribute representation.
Based on this two-stage decomposition, we can identify two kinds of errors. 
The first kind, related to the training phase, is due to domain shift.
The map from items to the attribute space that is trained on the seen classes might not generalize accurately to the unseen classes.
This contribution to the error can be arbitrarily large without introducing strong assumptions, as no labeled data is available for the unseen classes.
Thus it is impossible to characterize the domain shift between seen and unseen classes.
The second kind, related to the inference phase, is due to the fact that the class-attribute matrix might not fully differentiate among the unseen classes. 
In particular, there can be an item of the unseen classes with a set of attributes that according to the class-attribute matrix relation conforms with the description of two different classes.
%\alessio{Since the class-attribute matrix provides probabilities measure, the verb conforms could be maybe hard to interpret? maybe:  "In particular, there could possibly be many items that share the same attribute representation even if they belong to two different unseen classes". On a second read, actually I really prefer how it is written right now.}
The first kind of error, domain shift, has been extensively studied both in the theory and the experimental literature \citep{mansour2009domain,ben2010theory,Sener2016LearningTR,Pinheiro2018UnsupervisedDA,Luo2019TakingAC}.
%~\citemissing{}\alessio{paper for experimental?}.
In this work, for the first time, we theoretically characterize the contribution of the second kind of error. It is important to understand and characterize this error for specific ZSL tasks because it corresponds to an inherent information gap in the problem setting that cannot be circumvented with a smarter algorithm.
%As explained above, the first kind of error is hard to quantify in theory, because it depends on the arbitrary domain shift between the seen and the unseen classes. In this work, for the first time, we theoretically characterize the contribution of the second kind of error.

We provide tight lower and upper bounds on the worst-case error of the best map from attribute representations to classes based on the class-attribute matrix.
%This matrix summarizes the relations between unseen classes and attributes, and its entries are the probabilities to observe an attribute given an item of a class. The information given by this matrix is partial, as it does not provide the correlations between attributes.
Our analysis gives a lower bound in the sense that it bounds from below the minimum error that any method can guarantee given only the information of the class-attribute matrix. 
The class-attribute matrix specifies the fraction of items in each class that exhibit each attribute. There is a range of class-attribute distributions that satisfy the constrains defined by a given matrix. We give a lower bound on the error of the best possible method for the worst case distribution in that range.
This distribution represents a worst case correlation between attributes that satisfies the class-attribute matrix while maximizing the difficulty to distinguish  between attribute-representation of items belonging to distinct classes. 
Our analysis also gives an upper bound in the sense that we show a randomized classifier that achieves at most the error of the lower bound, assuming perfect item-to-attribute mapping.
This also shows that the lower bound is tight.
Interestingly, the value of the lower bound can also be interpreted as the quality of the information provided by the class-attribute matrix.
To the best of our knowledge, this is the first work to quantify such information.

\textbf{Contributions.}
Our main contributions are the following:
\begin{compactenum}
    \item We show the first non-trivial lower bound for attribute-based ZSL 
    %given a class-attribute matrix 
    (\Cref{sec:adv}).
    \item We formulate the lower bound given a class-attribute matrix as a linear program (\Cref{sec:compute_bound}).
    \item We show a closed form expression for the lower bound for binary classification (\Cref{binary-exact-computation}).
    %\item We show how to efficiently approximate the lower bound for multi-class classification based on the closed formula for binary classification (\cref{approximation-subsection}).
    \item We show that the lower bound is tight: we exhibit a randomized classifier 
    %from attributes to classes 
    whose expected error is upper bounded by the value of the lower bound (\Cref{subsection:tight}).
    \item We run extensive experiments comparing the theoretical results with the error of popular attribute-based ZSL methods, on benchmark datasets. We show that information given by the bound can be predictive of how standard methods behave, including which classes will likely be confused with others (\Cref{sec:experiments}).
\end{compactenum}

\section{Background and Related Work}
\label{sec:related-work}

Much early work on ZSL focused on using logical descriptions of the classes as auxiliary information, including attributes~\citep{chang:aaai08, lampert2009learning}.
Since then, an increasing number of ZSL methods have been proposed, which differ in methodology and the auxiliary information they use.
Examples of auxiliary information are symbolic descriptions of classes~\citep{chang:aaai08, lampert2009learning}, pre-trained embedding of the classes~\citep{frome:neurips13}, natural language descriptions~\citep{obeidat:naacl19,brown:neurips20}, and knowledge graphs \citep{wang:cvpr18,kampffmeyer:cvpr19,nayak:arxiv20}.
Recent ZSL methods can be grouped into two main categories: embedding-based and generation-based~\citep{pourpanah:survey2020}.
Seminal embedding-based works used two-layer neural networks to link the image feature space to the semantic one~\citep{Socher2013ZeroShotLT}.
Later, they evolved into deep neural networks that either map semantic features into the visual space~\citep{Ba2015PredictingDZ,Zhang2017LearningAD,Changpinyo2017PredictingVE} or project both the image and semantic features into the same space~\citep{zhang:cvpr15,radford:icml21}.
Generative-based approaches employ various kind of Generative Adversarial Networks (GANs)~\citep{Mirza2014ConditionalGA} to synthesize the features of the unseen classes, and use them to train a ZSL classifier in a supervised fashion~\citep{Felix2018MultimodalCG,Li2019LeveragingTI,Xian2018FeatureGN,Xian2019FVAEGAND2AF,Narayan2020LatentEF}.

ZSL with attributes generally consists of learning a linear map from the item to the attribute space, in the first stage. 
%In the second stage
Then, we use the class-attribute matrix to infer the correct class given the item-attribute representation~\citep{xian2018zero}.
%In~\cref{sec:data}, we provide more in-depth descriptions of such methods.
ZSL with attributes can be seen as a special case of embedding-based ZSL, in which the class embeddings are the rows of the class-attribute matrix.
Analyzing more general embedding-based or generation-based ZSL methods is challenging because they rely on deep neural networks for which relatively little theory is available.
%Although embedding- and generative-based methods are able to solve tasks, that can even be hard for humans, they entirely rely on complex deep neural networks for which very limited theory is provided.

%The first attribute-based ZSL methods find their root in work that describs how different object categories can share features \cite{bart2005cross}.
    %SEE relate work of https://www.cs.toronto.edu/~hinton/absps/palatucci.pdf
%\citet{farhadi2009describing} and \citet{lampert2009learning} are among the first to describe how to use attributes to recognize novel classes by respectively detecting learned properties of new images, and by providing an attribute description of a new image. 

%The first proposed method for attribute based ZSL is by \citet{Lampert2014AttributeBasedCF}, and in the following years many other methods have been proposed \cite{}.

%The sheer amount of different ZSL methods in the literature is backed by very little theory work.
Inspired by previous work to describe classes using error-correcting output codes~\citep{dietterich1994solving},~\citet{palatucci2009zero} were the first to propose a ZSL algorithm for which they can provide a theoretical analysis.
The algorithm learns linear classifiers individually for each binary attribute, and the attributes are mapped to the closest class-attribute representation.
While they are able to provide a PAC bound, their analysis relies on several strong assumptions that limit the problem setting.
First, they assume that they can learn each attribute independently, but attribute dependency is a widely recognized problem for attribute detection \citep{Jakulin2003AnalyzingAD}.
Second, they assume that each class has a unique attribute representation, i.e. each attribute must be either present or not in all the items of a given class.
Finally, they also assume that they are able to sample classes from a given distribution, and they are able to generalize to the non-sampled classes. That is, they do not separate beforehand between seen and unseen classes, which is the common scenario observed in ZSL settings.
Conversely, our lower bound does not assume a unique binary representation for each class, as we are given a class-attribute matrix that provides the probabilities to observe an attribute given an item of a class. 
Also, our lower bound takes into account the possible correlation between attributes, and it is computed based on the information provided on the given unseen classes. 

In more recent work, \citet{RomeraParedes2015AnES} draw a connection between transfer learning \citep{ben2010theory} and ZSL to provide a novel theoretical result.
In particular, they show that their model is not able to generalize if the attribute representations of the seen classes are orthogonal to the one of the unseen classes.
Intuitively, if those representations are orthogonal, the attribute map learned for the seen classes would fail to provide information for the unseen classes.
This is an impossibility result, and it is not able to arbitrarily quantify the information given for the unseen classes.
Unfortunately, transfer learning or domain adaptation like-bounds are challenging to estimate in a ZSL setting.
In fact, a term of those bounds require access to labeled data for the unseen classes, which is unavailable in ZSL. 
Another term, the discrepancy, depends on the difference between the attribute representations of the classes and the distribution of the items between seen and unseen.
While it would be theoretically possible to compute the discrepancy based on the information available, its computation is very challenging and it has been possible only in very specific cases~\citep{mansour2009domain}. 
%\alessio{In \cite{RomeraParedes2015AnES},  they show that the discrepeancy is $1$ in the extreme case of orthogonality, hence we cannot learn}.

Our novel lower bound is developed using adversarial techniques that describe the worst-case scenario with respect to the information available. It is inspired by recent work on semi-supervised learning, where the goal is to use the information provided by weak supervision sources~\citep{balsubramani2015optimally, arachie2021general, mazzetto2021semi, mazzetto2021adversarial}. The adversarial approach allows us to handle the possible dependencies between the attributes.

\section{Preliminaries}
\label{preliminaries}

We denote scalar and generic items using lowercase letters, vectors using lowercase bold letters, and matrices using bold uppercase letters. Given two vectors $\bm{v}$ and $\bm{v}'$, we denote with $\bm{v}\bm{v}'$ the concatenation of the two vectors. For any $n \in \mathbb{N}$, we denote with $[n]$ the set $\{ 1, \ldots, n\}$. Due to space constraints, all proofs are deferred to the appendix.

Let $\D$ be a distribution defined over the \textit{classification domain} $\domain$. A \textit{multiclass classification task} is specified by a \textit{labeling function} $y: \domain \rightarrow \mathcal{Y} = [k]$ that maps each \textit{item} $x \in \domain$ to a class $j$ in the label space $\mathcal{Y}$, where $k \geq 2$. We say that a multiclass classification task is \textit{balanced} if for each $j \in [k]$, it holds that $\Pr_{x \sim \mathcal{D}}[ y(x) = j] = 1/k$. Unless otherwise stated, we assume that the classification task is balanced.
This assumption is not restrictive, and as we will observe later, it can be changed if a different prior is known on the class probabilities. We will show that our lower bound holds even if we do not assume balanced classes.
We also assume to have access to $n$ \textit{attribute functions} $\fatt{1},\ldots,\fatt{n}$, where $\fatt{i}: \domain \rightarrow \{0,1\}$ for $i \in [n]$. We say that a classification item $x \in \domain$ has attribute $i \in [n]$ if $\fatt{i}(x) = 1$. For ease of notation, we define $\bpsi(x) \doteq ( \fatt{1}(x), \ldots, \fatt{n}(x))^T$. The codomain of $\bpsi$ is $\{0,1\}^n$, and it is referred to as \textit{attribute space}.
 All the information about the target unseen classes available to the algorithm is encoded in a \textit{class-attribute matrix} $\attmat \in [0,1]^{k \times n}$. The matrix provides information on the relations between classes and attributes.  In particular for a class $j \in [k]$, and an attribute $i \in [n]$,  $A_{j,i}$ is the probability that $\fatt{i}(x) = 1$ given that $y(x) = j$, i.e.,
\begin{align}
\label{matrix-class-attribute-relation}
    A_{j,i} = \Pr_{x \sim \mathcal{D}}[ \psi_i(x) = 1 | y(x) = j] \enspace .
\end{align}
An \textit{attribute-class classifier} $g$ is a map from vectors in the attribute space to classes, i.e., $g : \{0,1\}^n \rightarrow [k]$. The error of $g$ is $
    \eps( g ) \doteq \Pr_{x \sim \mathcal{D}}[g \circ \bpsi(x) \neq y(x)]$.
Let $\mathcal{G}$ be a collection of all the possible deterministic maps $\{0,1\}^n \rightarrow [k]$ from the attribute space to the $k$ classes. We are interested in evaluating $\min_{g \in \mathcal{G}} \eps(g)$.
As we focus on the contribution of the information provided by the class-attribute matrix, we assume access to the attribute functions $\psi_1,\ldots,\psi_n$. In practice, the map to the attribute space is learned on the available labeled data for the seen classes \citep{Lampert2014AttributeBasedCF,RomeraParedes2015AnES}, %and they are likely to be noisy, and thus be cause of additional error.
and it is likely noisy, and can be cause of additional error.

%In practice, these mappings are noisy, but our lower bounds hold regardless of the noise level in the mapping.

% Observe that in practice, we do not have access to the attribute functions $\psi_1,\ldots,\psi_n$, and any learned approximation of those functions will likely present an error. As we want to restrict our focus on evaluating the information provided by the class-attribute matrix, we ideally assume that we are mapping from the perfect attribute functions $\bpsi$ to the classes $[k]$; and we do not take into account the additional error derived by only having access to an approximation of those functions.

Let $p^*$ be the (unknown) probability mass function (PMF) of the random vector $(\psi_1(x), \ldots, \psi_n(x), y(x))$ where $x \sim \mathcal{D}$. The support of $p^*$ is $\{0,1\}^n \times [k]$. For $\bm{v} \in \{0,1\}^n$, and $j \in [k]$, let
%\begin{align}
%\label{def-p*}
    $p^*(\bm{v},j) \doteq \Pr_{x \sim \mathcal{D}}[ \bpsi(x) = \bm{v} \land y(x) = j]$.
%\end{align}
The error of $g$ is a function of $p^*$:
\begin{align}
\label{error-of-a-distribution}
    \eps(g) = \eps(g,p^*) \doteq 1 - \sum_{\bm{v}\in \{0,1\}^n} p^*(\bm{v},g(\bm{v}))
\end{align}
A function $g^* \in \mathcal{G}$ that attains minimum error $\eps(g^*) = \min_{g \in \mathcal{G}}\eps(g)$ is a Bayes optimal classifier with respect to $p^*$, i.e. for each $\bm{v} \in \{0,1\}^n$, we have that $g^*(\bm{v}) = \argmax_{j \in [k]} p^*(\bm{v},j)$. Thus, 
% $\min_{g \in \mathcal{G}} \eps(g)$ can be computed as a function of $p^*$:
\begin{align}
\label{optimal-minimum}
    \min_{g \in \mathcal{G}} \eps(g) = 1 -  \sum_{\bm{v}\in \{0,1\}^n} \max_{j \in [k]} p^*(\bm{v},j) \enspace .
\end{align}
We do not have access to labeled data for the unseen classes, so we cannot estimate $p^*$.
Instead, we construct a lower bound with respect to the set of all distributions that fit the available information.
%provided by the matrix $\attmat$.
% provides partial information on the structure of $p^*$ that we would like to use in order to provide a worst-case lower bound to the value $\min_{g \in \mathcal{G}} \eps(g)$.

\section{Lower Bounds for Zero-Shot Learning with Attributes}\label{sec:adv}
In this section, we formally define our lower bound. Consider a PMF $p$ with support over $\{0,1\}^n \times [k]$. We say that $p$ satisfies the class-attribute matrix $\bm{A}$ if (as constraints \eqref{matrix-class-attribute-relation}) for each $i \in [n]$ and $j \in [k]$, 
\begin{align}
\label{matrix-constraint}
    \sum_{\substack{\bm{v} \in \{0,1\}^n : \\v_i = 1} }p(\bm{v},j) = A_{j,i} \sum_{\bm{v} \in \{0,1\}^n} p(\bm{v},j) \enspace .
\end{align}
Recall that $p$ is balanced if for each $j \in [k]$, it holds that $\sum_{\bm{v}} p(\bm{v},j) = 1/k$. 
Let $\mathcal{P}({\bm{A}})$ be the set of all possible PMFs $p$  with support over $\{0,1\}^n \times [k]$ that satisfy  \eqref{matrix-constraint} and are balanced.  Clearly, the unknown true distribution, $p^* \in \mathcal{P}(\bm{A})$. 
%where $p^*$ is defined as in \eqref{def-p*}. 
The set $\mathcal{P}({\bm{A}})$ can be interpreted as the collection of all the PMFs of the random vector $(\psi_1, \ldots, \psi_n, y)$ that satisfy the constraints imposed by the information available on the prediction task and on the attribute functions. While the matrix $\bm{A}$ provides precise information on the correlation between any pair of attribute function and class, it fails to provide information on the correlation between attribute functions, i.e., it does not fully specify the distribution $p^*$. Without additional information, any PMF in $\mathcal{P}(\bm{A})$ could be equal to $p^*$.
Similarly to $\eqref{error-of-a-distribution}$, given a PMF $p \in \mathcal{P}(\bm{A})$ and an attribute-class classifier $g \in \mathcal{G}$, we can define the error of $g$ with respect to the distribution $p$ as
\begin{align}
\label{error-distribution-general}
    \eps(g,p) \doteq 1 - \sum_{\bm{v}\in \{0,1\}^n}  p(\bm{v},g(\bm{v})) \enspace .
\end{align}
Following the computation in $\eqref{optimal-minimum}$, the error of the best map from attributes to classes with respect to $p \in \mathcal{P}(\bm{A})$ is computed as
\begin{align}
\label{Q-computation}
    Q(p) \doteq   \min_{g \in \mathcal{G}}\eps(g,p) = 1 -  \sum_{\bm{v}\in \{0,1\}^n} \max_{j \in [k]} p(\bm{v},j) \enspace .
\end{align}
We are interested in the quantity
\begin{align}
\label{adversarial-quantity}
    Q \doteq \max_{p \in \mathcal{P}(\bm{A})}Q(p)   
\end{align}
i.e., $Q$ is the maximum over all distributions $p \in \mathcal{P}(\bm{A})$ of the error of the best algorithm for distribution $p$.
In other words, $Q$ is the worst Bayes error with respect to all the distributions that satisfy the constraints imposed by the class-attribute matrix and on the class probabilities.
% among all the errors of the best map from attributes to classes, where these errors are computed with respect to all the distributions $p \in \mathcal{P}(\bm{A})$.
Since $p^*$ can be any vector in $\mathcal{P}(\bm{A})$, the value $Q$ represents a lower bound to the best error rate that an algorithm can guarantee.
% provide with the available information.a \textit{worst-case lower bound} (or adversarial) to the best error of a map from attributes to classes. In fact, 
In fact, without  further information on the attribute functions or the prediction task, it is possible that $p^*$ attains the maximum of $\eqref{adversarial-quantity}$, that is in the worst-case we have that
\begin{align*}
    \eps(g,p^*) = \eps(g) \geq Q  \hspace {10pt} \forall g \in \mathcal{G} \hspace{5pt}
\end{align*}
In other words, the quantity $Q$ reflects a worst-case scenario where the attribute functions are correlated in such a way that it is 
%adversarially 
hard to distinguish between the classes, even if the attribute functions still  satisfy the constraints $\eqref{matrix-class-attribute-relation}$ given by the class-attribute matrix $\bm{A}$.
In \cref{subsection:tight}, we show that this lower bound is tight. In particular, we prove that there exists a randomized classifier from the attribute space $\{0,1\}^n$ to the classes $[k]$ whose expected error is at most $Q$ with respect to any distribution $p \in \mathcal{P}(\bm{A})$. 
%This implies that if $Q$ is small, then it is always possible to compute a attribute-class classifier $g$ that achieves a (expected) small error $\leq Q$ with respect to $p^*$, and this computati

\textbf{Example.} Consider a balanced binary classification task with two attributes. The class-attribute matrix  $\bm{A} \in \mathbb{R}^{2\times 2}$ is such that $A_{i,j}=1/2$ for $i,j \in \{1,2\}$. Based on this class-attribute matrix, we consider two different scenario. In the first scenario (\textit{best-case}), we have that items from the first class have either both attributes or none with probability $1/2$, and items from the second class have only either the first attribute  or the second attribute with probability $1/2$. In this case, we can simply count the number of attributes that an item has to assign it to the correct class. In the second scenario (\textit{worst-case}), each item has either both attributes or none with probability $1/2$ independently from the item class. In this case, any mapping from the attributes to the class is going to incur an error of $1/2$.

\subsection{Computing the Lower Bound}\label{sec:compute_bound}
In this subsection, we show how to compute $Q$ as in $\eqref{adversarial-quantity}$ through a Linear Program (LP). To describe a generic PMF $p$, we introduce $2^n \times k$ variables $q_{\bm{v},j}$ with $\bm{v} \in \{ 0, 1\}^n$ and $j \in [k]$. We use additional $2^n$ auxiliary variables $\lambda_{\bm{v}}$, for $\bm{v} \in \{0,1\}^n$, to denote the maximums of \eqref{Q-computation}, i.e. $\lambda_{\bm{v}} = \max_{j \in [k]} q_{\bm{v},j}$. The LP is formulated as follows.
\begin{align}
    \label{lp-maximum}
    &1-Q  = \min \sum_{\bm{v}} \lambda_{\bm{v}}& \\
    &(a)  \sum_{\substack{ \bm{v} \in \{0,1\}^n :\\ v_i = 1}}q_{\bm{v},j} = A_{j,i}\sum_{\bm{v} \in \{0,1\}^n}q_{\bm{v},j}  & \forall j \in [k], i \in [n] \nonumber  \\
    &(b)  \sum_{\substack{ \bm{v} \in \{0,1\}^n }}q_{\bm{v},j} = \frac{1}{k} & \forall j \in [k] \nonumber  \\
    &(c) \hspace{5pt} \lambda_{\bm{v}} \geq  q_{\bm{v},j} \geq 0 & \forall \bm{v} \in \{0,1\}^n, j \in [k] \nonumber 
\end{align}    

\begin{theorem}
\label{lp-theorem}
The optimal value of the LP \eqref{lp-maximum} is equal to $1-Q$, with $Q$ is as in \eqref{Q-computation}.
\end{theorem}

By removing or modifying constraint $(b)$ of the LP, it is possible to remove the assumption that the classes are balanced or provide different class weights. All the previous results still hold by changing the definition of $\mathcal{P}(\bm{A})$ accordingly. It is important to point out that since we are computing a worst-case lower bound, the class weights provide significant information. Without constraints on the class weights, the worst-case distribution could concentrate all the probability mass on few classes that are hard to differentiate using the available class-attribute matrix $\bm{A}$. 

The LP has $O(k \cdot 2^n)$ variables and constraints, and therefore it is computationally expensive for large number of attributes. The dependency on $2^n$ is required to describe all the possible correlations between the output of the $n$ attribute functions. Nevertheless, we  present an efficient computation for the binary case in the next subsection, and an efficient approximation for the general case in \Cref{app:approximation-subsection}.

\subsection{Lower Bound for Binary Classification}
\label{binary-exact-computation}
In this subsection, we show how to efficiently compute $Q$ as in \eqref{adversarial-quantity} in the case of a binary classification task, i.e. $k=2$ and $\bm{A} = [0,1]^{2 \times n}$. For ease of notation, let
\begin{align}
\label{binary-matrix}
    \bm{A} =      \begin{bmatrix}
    \alpha_1 & \ldots & \alpha_n \\
    \beta_1 & \ldots & \beta_n 
  \end{bmatrix}  \enspace .
\end{align}

\begin{theorem}
\label{binary-adversarial-computation}
Consider a balanced binary classification task and let $\bm{A}$ be as in \eqref{binary-matrix}. Let $Q$ be as in \eqref{adversarial-quantity}. It holds that
    $Q = \frac{1}{2}\left(1 - \max_{i \in [n]}| \beta_i - \alpha_i|\right)$. Moreover, let $g_a$ be the attribute-class classifier
\begin{align*}
    g_a(\bm{v}) = 1+\begin{cases}
    v_{i^*} &\mbox{ if } \alpha_{i^*} < \beta_{i^*} \\
    1 - v_{i^*} &\mbox{ if } \alpha_{i^*} \geq \beta_{i^*}
    \end{cases}
\end{align*}
for each $\bm{v} \in \{0,1\}^n$, where $i^* = \argmax_{i}|\beta_i - \alpha_i|$ and $v_i$ is the $i$-th component of the vector $\bm{v}$. Then
    $\eps(g_a,p) = Q$ for all  ${p} \in \mathcal{P}(\bm{A})$,
i.e. the lower bound $Q$ is tight.
\end{theorem}
The theorem shows that in the worst-case, the attributes could be correlated in such a way that it is not possible to do better than deciding solely based on the attribute with the largest gap between its probabilities in the  two classes. This result also formally proves that for binary classification, the worst-case is determined by a single attribute, and there is no compounded benefit in having multiple attributes in the case of perfect attribute detectors. This result is in line with other worst-case analyses in the context of weak supervision. In \citet{mazzetto2021semi}, it is noted that while combining the output of different weak supervision sources to obtain a noisy label of a given input item, in the worst-case one cannot do better than just using the most accurate  weak supervision source without additional information. In \Cref{app:approximation-subsection}, we show how to approximate the lower bound in the multiclass setting by using \Cref{binary-adversarial-computation}.

\subsection{Lower Bound is Tight}\label{sec:tight}

In this subsection, we prove that the worst-case lower bound $\eqref{adversarial-quantity}$ is tight. We show  a randomized attribute-class classifier whose expected error is upper bounded by $Q$ with respect to any distribution $p \in \mathcal{P}(\bm{A})$. This classifier can be computed only based on the class-attribute matrix, and it provides an upper bound to the error of the best map from attributes to classes that matches the lower bound.

We consider the family $\mathcal{G}_R$ of all randomized  attribute-class classifiers, where each $g \in \mathcal{G}_R$ is a random map from $\{0,1\}^n$ to $[k]$. A attribute-class classifier in $\mathcal{G}_R$ is described with a right-stochastic matrix $\bm{W} \in [0,1]^{2^n \times k}$, where the rows are indexed by binary vectors $\bm{v} \in \{0,1\}^n$, and the columns are indexed by the classes $j \in [k]$. The entry $W_{\bm{v},j}$ represents the probability of the randomized classifier to output $j$ given that the input is $\bm{v}$. We will use $g_{\bm{W}}$ to denote the randomized classifier in $\mathcal{G}_R$ that is described with the right-stochastic matrix $\bm{W}$. Given a PMF $p$ over $\{0,1\}^n \times [k]$, we define the expected error of $g_{\bm{W}}$ as
\begin{align}
\label{expected-loss-randomized}
    \eps(g_{\bm{W}}, p) &\doteq 1 - \Pr_{(\bm{v},j) \sim p}[ g_{\bm{W}}(\bm{v}) = j] 
    = 1-\sum_{\bm{v} \in \{0,1\}^n} \sum_{j \in [k]} W_{\bm{v},j} \cdot p(\bm{v},j)    \enspace .
\end{align}
We can observe that the definition above extends definition $\eqref{error-distribution-general}$, in fact $\eqref{expected-loss-randomized}$ coincides with $\eqref{error-distribution-general}$ if $g_{\bm{W}}$ is a deterministic classifier, i.e. each row of $\bm{W}$ contains a $1$.
 
\label{subsection:tight}
\begin{theorem}\label{th:adv}
\label{minimax-invertible-theorem}
There exists a randomized attribute-class classifier $g_a \in \mathcal{G}_R$ such that its worst-case expected error is upper bounded by $Q$, i.e.
    $\max_{p \in \mathcal{P}(\bm{A})} \eps(g_a, p) \leq Q$ , where $Q$ is computed as in $\eqref{adversarial-quantity}$. Also,
it holds
       $\max_{p \in \mathcal{P}(\bm{A})} \min_{g \in \mathcal{G}_R}  \eps(g, p) = Q$,
i.e. the lower bound $Q$ also applies to the family of randomized functions $\mathcal{G}_R$.
\end{theorem}
It is possible to compute the randomized attribute-class classifier $g_a$ that satisfies \cref{minimax-invertible-theorem} solely based on the matrix $\bm{A}$ through Linear Programming using $O(k \cdot 2^n)$ variables and constraints. Due to space constraints, we defer this computation to Appendix~\ref{adversarial-classifier-computation}.

\section{Empirical Applications}\label{sec:experiments}
In this section we compare our novel theory with the performance of popular attribute-based ZSL methods.
Our results quantify a lower bound to the lowest error rate that any attribute-based ZSL algorithm can guarantee based on the information provided by the class-attribute matrix. In practice,
%In this section, we show how the developed theory compares with applications of popular attribute-based ZSL methods.
%Our lower bound quantifies the information provided by the class-attribute matrix, and it is impossible to guarantee a smaller error with the information available.
%ZSL models can still achieve better performance than the lower bound, that represents a worst-case scenario of the relations between classes and attributes.  
we show that the lower bound is still predictive of the performance and the behaviour of attribute-based ZSL algorithms.
We run two set of experiments.

\begin{compactenum}

\item \textbf{Comparing the lower bound and the empirical error} %\textbf{Relation between the lower bound and the empirical error} 
(\cref{sec:exp:adversarial}).
We 
%investigate how 
compare the error rates of ZSL models with attributes %compare 
with the lower bound on the error from \cref{sec:compute_bound}.

\item \textbf{Pairwise misclassification prediction} (\cref{sec:exp:misclassification}). 
%\eli{We measure the predictive power of our lower bounds in identifying pairs of classes that ZSL models are likely to  misclassify. This hardness is measured using the lower bound on the error for a pair of classes (\cref{binary-exact-computation}).}
We measure the predictive power of our lower bounds to
identify pairs of classes that ZSL models are likely to  misclassify. This hardness is measured using the lower bound on the error for a pair of classes (\cref{binary-exact-computation}).
%We investigate if ZSL models are likely to misclassify pair of classes that are hard to distinguish according to the theory. This hardness is measured using the lower bound on the error for a pair of classes (\cref{binary-exact-computation}).
%ZSL error predictions using the lower Bound
% Relation between the lower bound and the empirical error.
% Evaluation of the Lower Bound

\end{compactenum}
\subsection{Experimental Setup}\label{sec:data}

In this section, we briefly describe the experimental setup. Further details about the datasets and the methods can be found in \cref{app:exp}.
%Due to space constraints, we defer the description of each model to~\ref{app:exp}.
% and supplementary material.
%Code is available as supplementary material.
\footnote{Code is available at \url{https://github.com/BatsResearch/mazzetto-neurips22-code}.} 
%\subsubsection{Datasets} 
We choose the following four datasets with attributes that are widely used benchmarks in ZSL: Animals with Attributes 2 (\textbf{AwA2})~\citep{xian2018zero}, aPascal-aYahoo (\textbf{aPY})~\citep{farhadi2009describing}, Caltech-UCSD Birds-200-2011 (\textbf{CUB})~\citep{WahCUB_200_2011}, and SUN attribute database (\textbf{SUN})~\citep{patterson2014sun}.
%Due to space constraints, we report the details in~\cref{app:data}.
%For all these datasets, we use the split between seen and unseen classes suggested by~\citet{xian2018zero}.
%Except for AwA2 (see \Cref{app:data}), we obtain the class-attribute matrices by averaging the attribute representation of the images of each class. This is the same strategy used by~\citet{RomeraParedes2015AnES} in their experiments.
%For each dataset, we use a pre-trained ResNet-101~\cite{He2016DeepRL} as an encoder to extract features from the images.
%The features are $2048$-dimensional, and they are used as input for the ZSL models.
We focus on classic ZSL algorithms with attributes that 
use at most 
the information in
%provided by 
the class-attribute matrix for the unseen classes: \textbf{DAP}~\citep{Lampert2014AttributeBasedCF}, \textbf{ESZSL}~\citep{RomeraParedes2015AnES}, \textbf{SAE}~\citep{Kodirov2017cvpr}, \textbf{ALE}~\citep{Akata2016}, \textbf{SJE}~\citep{Akata2015cvpr}.
%We choose such methods to reduce the impact of other confounding factors that can contribute to determine the performance of the algorithms. 
We choose these methods because they use the class-attribute matrix that is the focus of our theoretical analysis.
%We remark that these methods assume access to the class-attribute matrix for the unseen classes as defined in our theory in their original experiments. 
Many other ZSL methods have been proposed in recent years (see \Cref{sec:related-work}), but their comparison with our lower bound would be vacuous as they often use other source of auxiliary information on the unseen classes, and thus do not fit within our novel theoretical framework.
They are beyond the scope of this first analysis of ZSL.
However, we also run experiments on a more recent attribute-based method \textbf{DAZLE} ~\citep{Huynh-DAZLE:CVPR20} which takes advantage of additional information, i.e., attribute semantic vectors.

\subsection{Comparing Lower Bound and Empirical Error}
\label{sec:exp:adversarial}

%In this section, we compare the lower bound presented in~\cref{sec:adv} with the actual error of the ZSL models discussed in~\ref{sec:zsl_methods}.
In this section, we compare the lower bound presented in~\cref{sec:adv} with the actual error of the ZSL models.
To this end, we run two set of experiments: a first set using the ZSL datasets mentioned in the previous subsection,  and a second set using adversarially generated synthetic data that conform with the class-attribute matrices of those same ZSL datasets.

%In the first set, we simulate our theoretical framework. In particular, we train the ZSL methods on the seen classes, and evaluate them on a held-out set of data from the same classes. 
%Consistently, 
%The lower bound is computed with respect to using the class-attribute matrix of the seen classes, and it quantifies the worst-case error to map from attributes to classes.
%By evaluating the ZSL models on instances from the same distribution of the seen classes, we minimize the additional error introduced by inaccurate image to attribute mappings.
In the first set of experiments, we follow the standard way to evaluate ZSL models. We train our model on the seen classes, and then compare our lower bound with the empirical error of the ZSL models on the unseen classes.
Since the computation of the lower bound is very expensive for a large number of attributes  (\cref{sec:compute_bound}), we focus on a subsets of them. 
We propose the following greedy strategy to ensure a selection of attributes that are informative with respect to the target classes. Starting with no attributes, we iteratively add the attribute that decreases the most the value of the lower bound, up to $15$ attributes. Due to the large number of seen classes of SUN and CUB, we restrict them to a smaller random subset (see \Cref{app:exp}). % among all the seen classes.
% In this setting, we cannot guarantee the mapping from image to attribute learned on the seen classes to be accurate on the unseen classes, due to the domain shift. This can be source of additional error that can significantly impact the performances of the ZSL methods. 
% As before, we select the attributes using the greedy strategy, and we report the results in the second row of~\cref{fig:err_attr}.
In the first row of~\cref{fig:err_attr}, we report results for aPY, AwA2, and CUB, due to space constraints. The results for SUN are similar and  in~\Cref{fig:err_attr_SUN} in~\cref{app:exp:res}.

We observe that the value of the lower bound can be significantly lower than the error rate of the ZSL models.
This gap is most probably due to the fact that the learned map from images to attributes does not generalize perfectly to the unseen classes. 
In fact, in this setting we can identify two main source of error for the ZSL models: (1) the arbitrary error due to the domain shift, and (2) the error due to how discriminative is the attribute space to differentiate between the different classes. Our lower bound only addresses the latter, as no method can guarantee a smaller error than the lower bound to map from attributes to classes given only the information of the class-attribute matrix. 
Nonetheless, for CUB and SUN we observe that the empirical error of ZSL models roughly follow the trend of the lower bound. 
This suggests that the lower bound can still be used as a tool to capture how the additional information provided by an attribute leads to improvements of the ZSL models.

\begin{figure*}[t]
    \centering
    \includegraphics[width=.326\columnwidth]{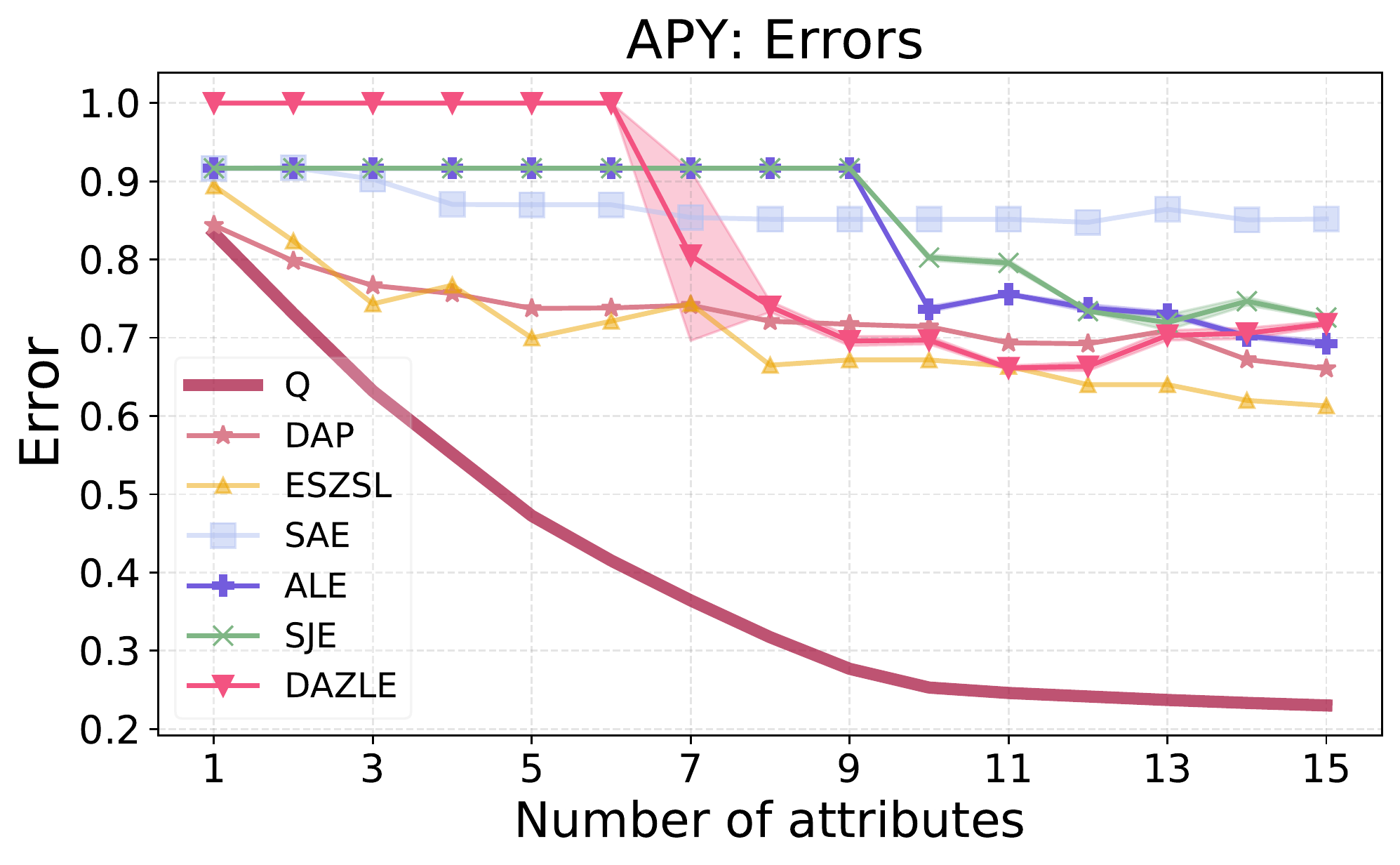}
    \includegraphics[width=.326\columnwidth]{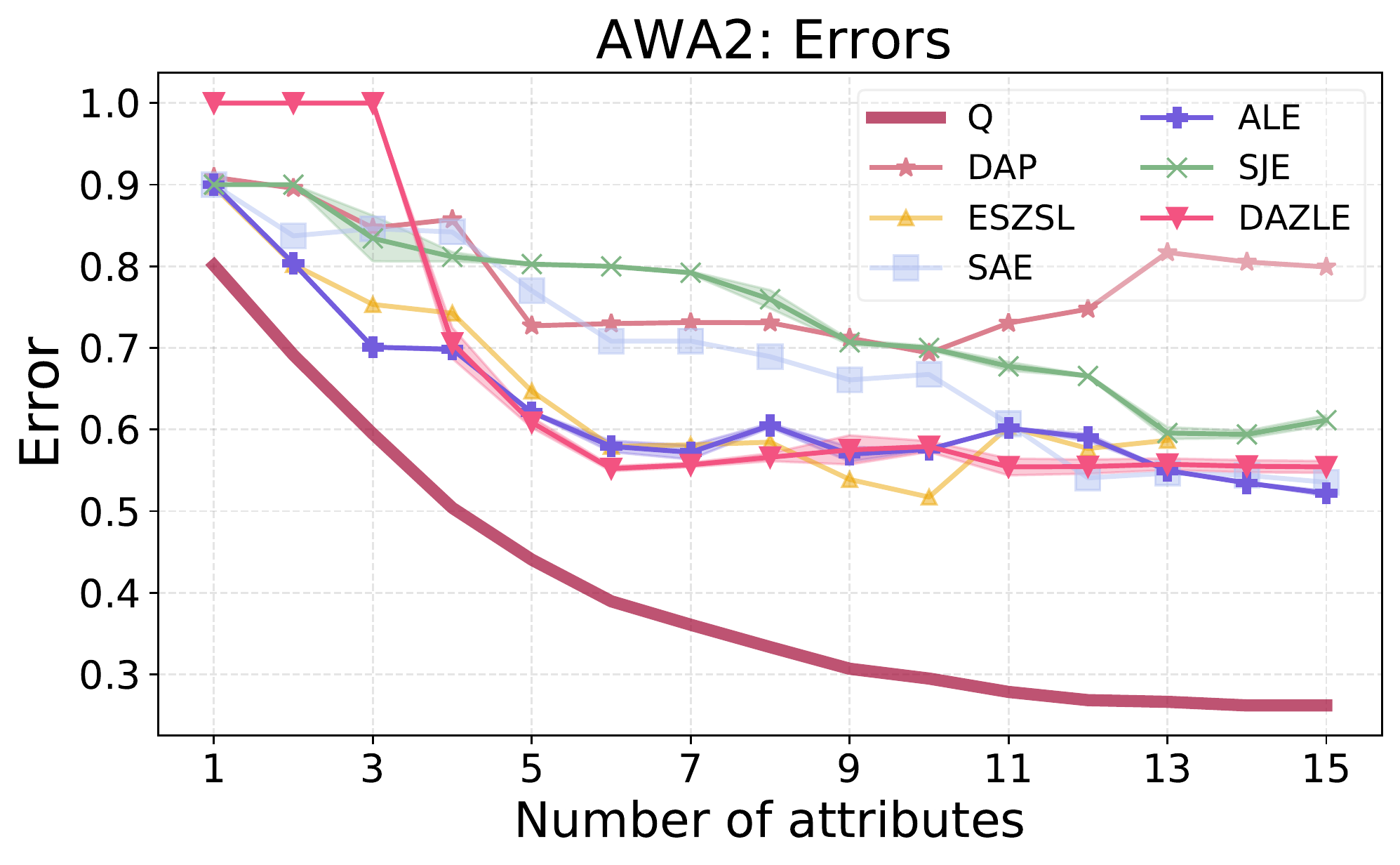}
    \includegraphics[width=.326\columnwidth]{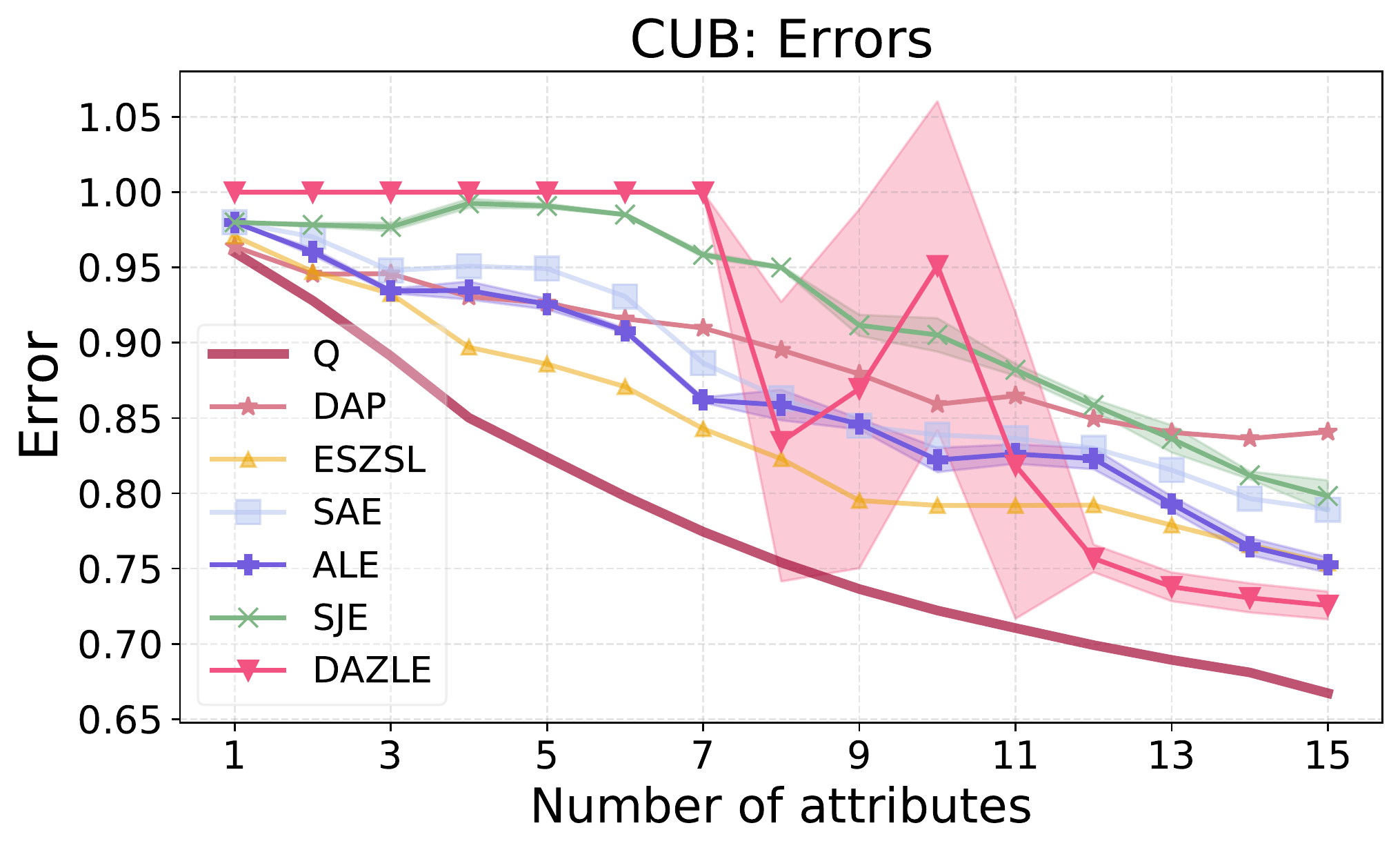}
    \includegraphics[width=.326\columnwidth]{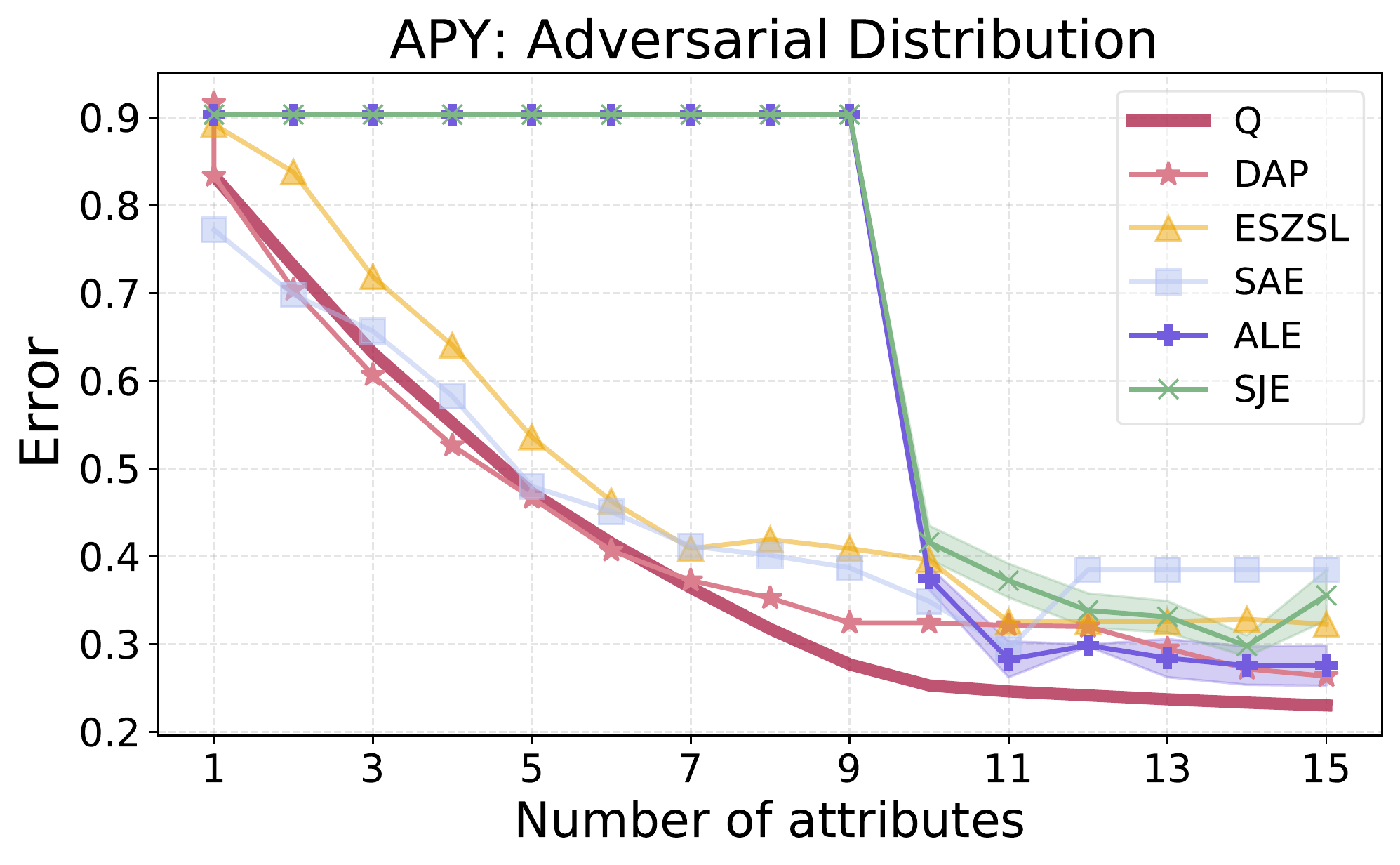}
    \includegraphics[width=.326\columnwidth]{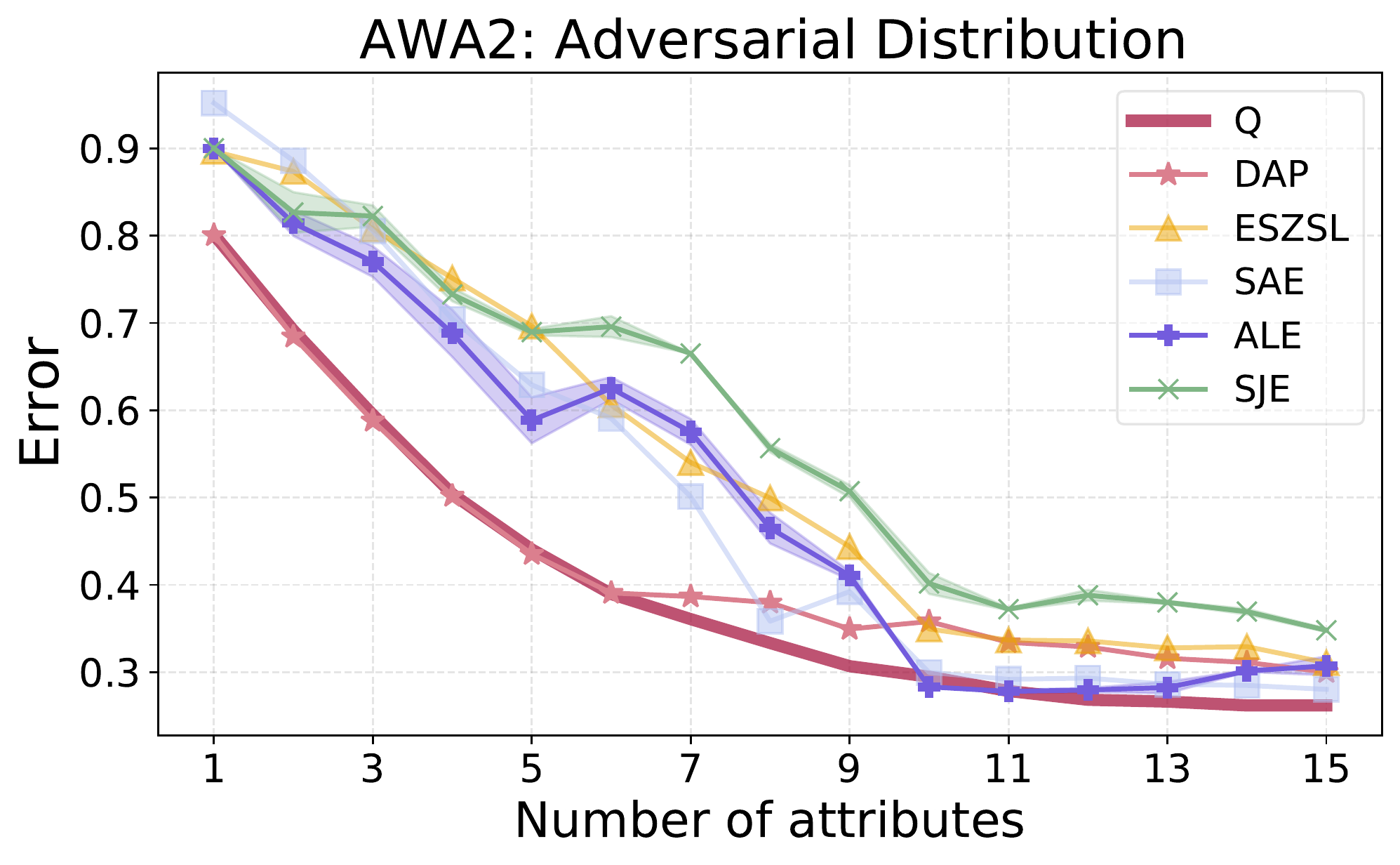}
    \includegraphics[width=.326\columnwidth]{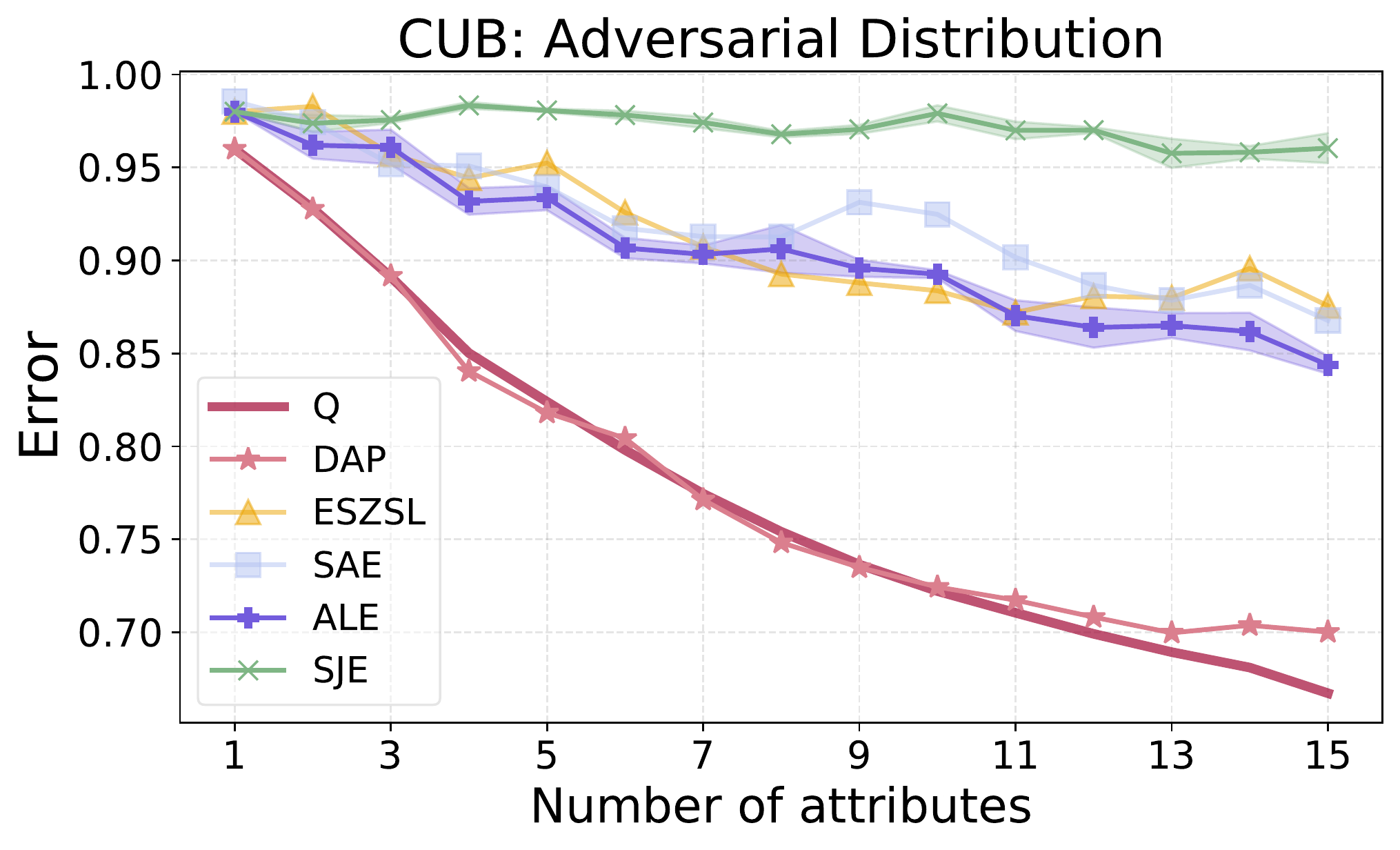}

    \caption{\textbf{Comparison of the lower bound with the empirical error.} We plot the lower bound on the error (\textbf{Q}), and the error of ZSL methods with attributes (\textbf{DAP, ESZSL, SAE, ALE}, and \textbf{DAZLE}).
    The first row reports these values computed on the unseen classes of the aPY, AwA2, and CUB, varying the number of available attributes. %The ZSL error can be much higher than $Q$ most probably because of the error introduced by the domain-shift between seen
    %and unseen classes.
    The second row reports the values for the adversarially generated synthetic data. 
    %~\cref{sec:exp:adversarial} provides detailed analysis of the results. 
    The bands indicate the standard errors on five runs with different seeds for randomized methods. 
    %We defer to~\cref{app:exp} the plots of SUN.
    These results validate that even in the absence of domain shift, there exists a distribution of the data that satisfy the constraints imposed by the class-attribute matrix for which no method can do better than the lower bound.
    }
    \label{fig:err_attr}
\end{figure*}

In the second set of experiments, we empirically demonstrate our theory by showing that even if we minimize the error due to domain shift, there exists data for which no method can do better than our lower bound. To this end, for each dataset we adversarially generate synthetic data with attribute values satisfying the dataset's class-attribute matrix. Specifically, we use the same class-attribute matrix with $15$ attributes as in the previous set of experiments in order to compute the adversarial distribution $p$ over attributes and classes according to the linear program introduced in \Cref{sec:compute_bound}. The data is generated by sampling attribute-class pairs from this distribution, and using the attribute vector as the feature vector. In order to minimize the error due to domain shift, this distribution is used to generate data for both  training and testing of the ZSL methods, and the same class-attribute matrix is used for both seen and unseen classes. We report additional details on this experimental setup and synthetic data generation in \Cref{app:exp}. 
We report the results of the experiments in the second row of~\cref{fig:err_attr}, iterating over the same attributes greedily selected in the first set of experiments for each dataset.
(For this set of experiments, we do not report results for DAZLE as this method relies on the input items being images, so it does not apply to our synthetic data.)
In this case, the methods are able to achieve errors that are comparable with the lower bound as we minimized the error due to domain shift. This experiment empirically validates that even in the absence of domain shift, there exists a distribution of the data that satisfy the constraints imposed by the class-attribute matrix for which no method can do better than the lower bound. This adversarial distribution represent an intrinsic error gap due to the quality of the information provided by the class-attribute matrix.
This is the first work to quantify such information in ZSL.

\subsection{Pairwise Misclassification Prediction}
\label{sec:exp:misclassification}

\Cref{binary-adversarial-computation} shows how to efficiently compute the lower bound on the error to distinguish between a pair of classes given the class-attribute matrix.
In addition to the overall bound on error, it also gives us fine-grained information about which classes are harder to distinguish among.
We define the \emph{pairwise lower bound error matrix} $\bm{L}$, whose entry $L_{j,j'}$ is the lower bound on the error computed as in~\cref{binary-exact-computation}, for all classes $j,j' \in [k]$, $j \neq j'$. %Conceptually, 
A large entry $L_{j,j'}$ between two classes $j\neq j'$ indicates that it is hard (in the worst-case) to distinguish between them.
%Following the construction described in \cref{approximation-subsection}, we can build a weighted complete graph $G$ on top of the classes $[k]$, where an edge between two classes is weighted according to the lower bound on the error computed as in \cref{binary-exact-computation}. Conceptually, a large weight $w_{j,j'}$ on an edge between two classes $j\neq j'$ indicates that it is hard (in the worst-case) to distinguish between them.
In this section, we compare the matrix $\bm{L}$ with the classification errors made by the ZSL models discussed in~\Cref{sec:data}. 
In particular, we want to show if the pairwise lower bounds on the errors are predictive of the misclassification errors made by the ZSL models. Specifically, a large lower bound on the error for a pair of classes indicates that a ZSL model would likely confuse one class with the other.
% Possible solution
%For a given dataset and a ZSL method, we build a misclassification matrix $\bm{M}$ whose entries $M_{j,j'}$, for all $j,j' \in [k]$, $j \neq j'$, represent the probability to misclassify an item of the class $j$ with the class $j'$ or vice-versa. This matrix is estimated using the test data for the unseen classes.
For a given dataset and a ZSL method, we build a \emph{misclassification error matrix} $\bm{M}$. The entry $M_{j,j'}$ is computed as
\begin{align*}
\Pr_{x \sim \mathcal{D}}( h(x) = j \land y(x) = j' | y(x) \in \{j ,j'\}) 
+ \Pr_{x \sim \mathcal{D}}( h(x) = j' \land y(x) = j | y(x) \in \{j ,j'\})
\end{align*}
for all $j,j' \in [k]$, $j \neq j'$, where $h(\cdot)$ is the ZSL model. The entry $M_{j,j'}$ represents the probability of the model to misclassify an item of the class $j$ with the class $j'$ or vice-versa. We estimate $\bm{M}$ using test data of the unseen classes. 
%This matrix is estimated using the test data for the unseen classes.

%We compute $\bm{L}$ using the class-attribute matrix of .
In~\Cref{fig:error_matrices}, we plot $\bm{L}$ together with the misclassification matrices $\bm{M}$ of two ZSL methods: DAP and ESZSL, computed on the unseen classes of aPY.
%We observe that t
The pairwise lower bound matrix $\bm{L}$ has large values within multiple groups of semantically similar classes, e.g., animals and transportation means. This is in line with human intuition, as we expect visually similar classes to exhibit similar attributes. Correspondingly, the misclassification matrices of DAP and ESZSL highlight the presence of many %misclassification 
errors for classes belonging to these groups.  We also note that ZSL models could misclassify other pairs of classes due to other source of errors, such as an inaccurate map from image to attributes. 
%In the remaining part of this section, we propose a way to analytically quantify the similarity between the pairwise lower bound error matrix $L$ and the misclassification matrices. 
We report additional experimental analysis on all datasets in \Cref{app:exp:res}.

\begin{figure*}[t]
    \centering
    \includegraphics[width=.314\columnwidth]{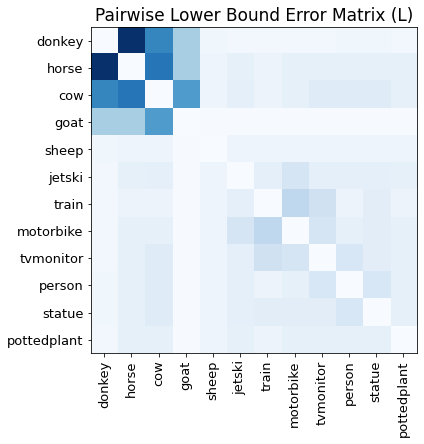}
    \includegraphics[width=.322\columnwidth]{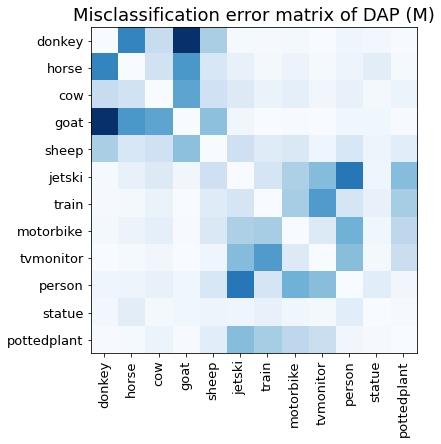}
    \includegraphics[width=.328\columnwidth]{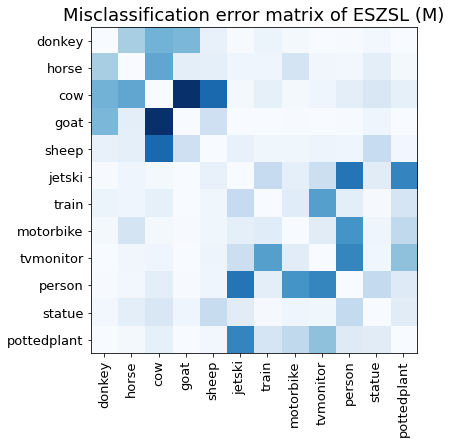}
    \caption{\textbf{Pairwise miscassification matrices.} For the unseen classes of aPY, we plot the pairwise lower bound between pair of classes $\bm{L}$ (\cref{binary-exact-computation}), and the misclassification error matrix $\bm{M}$ of two ZSL models: DAP and ESZSL. 
    Darker squares indicate higher values, and light blue on the diagonal is 0.
    High values of the lower bound indicate classes that are harder (in the worst-case) to distinguish in theory, and high values in $\bm{M}$ indicate pair of classes that are often confused by the ZSL model.
    %The color gradient is between $0$ and $1/2$ for the adjacency matrix of $G$, and between $0$ and $1$ for the misclassification matrices. 
    }
    \label{fig:error_matrices}
\vspace{-1em}
\end{figure*}

\section{Conclusions, Limitations, and Future Work}
We present the first non-trivial lower bound on the best error that an attribute-based ZSL method can guarantee given the information provided---the class attribute matrix. 
While our method is limited to class-attribute matrices, it constitutes a first theoretical building block to quantify the auxiliary information provided in ZSL. In general, theoretical evaluation of the error of ZSL models remains a hard problem due to the arbitrary domain shift between seen and unseen classes, and the wide range of possible auxiliary information used. As a future direction, it remains an open problem to be able to quantify this information for other families of ZSL methods.
However, our analysis readily extends to other variants of ZSL, such as generalized ZSL, where we simply use the class-attribute matrix of the union of both seen and unseen classes while computing our lower bound. \\
%Moreover, it also englobes methods that learn soft-predictors for attributes. 
%
%
%
%We prove our lower bound is tight, as we show a randomized classifier whose expected error is upper bounded by this value.
%Empirically, we demonstrate that the lower bound can be used to explain the behaviour of popular \alessio{attribute-based} ZSL methods.
%
%
\subsection*{Broader Societal Impacts}
% James sketch from call notes; someone who knows the paper needs to make this sound good.
Zero-shot learning is now a popular scenario in research, with potential application to real-world language and vision tasks. 
Worse-case guarantees have long been desired in ZSL.
Any improvement in the rigor of claims about model performance has impact because it demonstrates both what performance can be achieved and that some solutions are invalid. 
However, such bounds do not cover many kinds of error, such as a generalization gap from domain shift or label errors.
Further, it is important that bounds are correctly interpreted such that no false claims or confidences are drawn from our findings.
An educated interpretation of the effect of these bounds upon any particular machine learning application is still required.

\subsection*{Acknowledgements}
We thank Michael Littman and James Tompkin for many helpful and insightful discussions. This material is based on research sponsored by Defense Advanced Research Projects Agency (DARPA) and Air Force Research Laboratory (AFRL) under agreement number FA8750-19-2-1006 and by the National Science Foundation (NSF) under award
IIS-1813444. The U.S. Government is authorized to reproduce and distribute reprints for Governmental purposes notwithstanding any copyright notation thereon. The views and conclusions contained herein are those of the authors and should not be interpreted as necessarily representing the official policies or endorsements, either expressed or implied, of Defense Advanced Research Projects Agency (DARPA) and Air Force Research Laboratory (AFRL) or the U.S. Government. We gratefully acknowledge support from Google and Cisco. Disclosure: Stephen Bach is an advisor to Snorkel AI, a company that provides software and services for weakly supervised machine learning.

\bibliography{neurips_2022}
\bibliographystyle{neurips_2022}  

\newpage

\newpage

\appendix
\newcommand{\p}[2]{p^{(#1)}_{#2}}
\newcommand{\bv}{\bm{v}}
\section{Deferred Proofs.}
\label{deferred-proofs}
\textbf{Proof of \cref{lp-theorem}}.
Let $Q'$ be the optimal value of the LP.

Let $p^* \in \mathcal{P}(\bm{A})$ be a solution of the maximization $\eqref{adversarial-quantity}$. Consider the following assignment of the variables $q_{\bm{v},j} = p^*(\bm{v},j)$ for all $\bm{v} \in \{0,1\}^n$ and $j \in [k]$. Since $p^* \in \mathcal{P}(\bm{A})$, it is straight-forward to verify that the variables $q_{\bm{v},j}$ satisfy constraints $(a)$ and $(b)$ of the LP. Moreover, the objective function is minimized whenever the values $\lambda_{\bm{v}}$ are chosen as small as possible. Due to constraint  $(c)$ of the LP, we have that $\lambda_{\bm{v}} = \max_{j \in [k]} q_{\bm{v},j}$ for each $\bm{v} \in \{0,1\}^n$.
We have that
\begin{align}
\label{step-thm-4.2}
    1-Q = \sum_{\bm{v} \in \{0,1\}^n} \max_{\bm{v} \in \{0,1\}^n} \max_{j \in [k]} p^*(\bm{v},j)=\sum_{\bm{v} \in \{0,1\}^n} \lambda_{\bm{v}} \geq 1-Q'
\end{align}

By contradiction, assume the optimal solution $q^*_{\bm{v},j}$, $\lambda^*_{\bm{v}}$ is such that $1-Q'= \sum_{\bm{v} \in \{0,1\}^n} \lambda_{\bm{v}}^* < 1-Q$. Since $q^*_{\bm{v},j}$, $\lambda^*_{\bm{v}}$ is an optimal solution, due to constraint $(c)$ we have that $\lambda^*_{\bm{v}} = \max_{j \in [k]} q^*_{\bm{v},j}$. Consider a PMF $\tilde{p}$ over $\{0,1\}^n \times [k]$ such that $\tilde{p}(\bm{v},j) = q^*_{\bm{v},j}$. It is easy to verify that $\tilde{p} \in \mathcal{P}(\bm{A})$ due to the constraint $(a)$ and $(b)$. Moreover, we have that $Q(\tilde{p}) = 1 - \sum_{\bm{v}}\max \tilde{p}(\bm{v},j) = 1 - \sum_{\bm{v}} \lambda^*_{\bm{v}} > Q$. This is a contradiction as $\max_{p \in \mathcal{P}(\bm{A})}Q(p) = Q$. Therefore, we have that $1-Q' \geq 1-Q$. Combining the latter inequality with inequality \eqref{step-thm-4.2}, we can conclude that $Q = Q'$.

\qed

\medskip
\textbf{Proof of \cref{binary-adversarial-computation}}.
Without loss of generality, we assume that $\alpha_i \geq \beta_i$ for each $i \in [n]$.  In fact, if $\alpha_i < \beta_i$, then we can consider the attribute function $\psi'_i = 1 - \psi_i$, and the $i$-th column of the matrix $\bm{A}$ would become $(1-\alpha_i, 1-\beta_i)^T$, with $1 - \alpha_i \geq 1-\beta_i$. Also, assume that the attributes are ordered such that $\alpha_1 - \beta_1 \geq \alpha_i - \beta_i$ for each $i \in [n]$.

We first prove the second part of the Theorem. Let $g_a$ be defined as in the problem statement. It is easy to see that for any $p \in \mathcal{P}(\bm{A})$, we have that
\begin{align}
    \eps(g_a,p) &= \Pr_{x \sim \mathcal{D}}( g_a \circ \bpsi(x) \neq y(x)) = \label{greedycomputation}\\
    &=\Pr( \psi_1(x) = 0 | y(x) = 1) \Pr(y(x)=1)+ \Pr( \psi_1(x) = 1 | y(x) = 0) \Pr(y(x)=0) \nonumber \\
    &= \frac{1}{2}(1 - \alpha_1) + \frac{1}{2}\beta_1 \nonumber \\
    &= \frac{1}{2}( 1 - |\beta_1 - \alpha_1|) = Q \nonumber
\end{align}
Since this holds for any $p \in \mathcal{P}(\bm{A})$, we have that
\begin{align*}
    \max_{p \in \mathcal{P}(\bm{A})}\eps(g_a,p) = \frac{1}{2}( 1 - |\beta_1 - \alpha_1|) \enspace .
\end{align*}

Now, we will prove the first part of the Theorem. The proof is by induction. For $i \in [n]$, let $\bm{A}_i$ be the matrix that consists of the first $i$ columns of $\bm{A}$. For $i \in [n]$, let $\mathcal{G}_i$ be the set of all the functions $\{0,1\}^n \rightarrow [2]$.
For $i \in [n]$, let $p^{(i)}$ be a PMF with support over $\{0,1\}^i \times [2]$ such that
\begin{align}
\label{p-i-definition}
    p^{(i)} = \argmax_{p \in \mathcal{P}(\bm{A}_i)} \min_{g \in \mathcal{G}_i} \underbrace{\left( 1 - \sum_{\bm{v} \in \{0,1\}^i}p(\bm{v}, g(\bm{v})) \right)}_{\substack{\doteq \\ \eps^{(i)}(g,p) }}
\end{align}
For ease of notation, for $i \in [n]$, we will denote $\p{i}{\bv,j} = p^{(i)}(\bm{v},j)$ for each $\bm{v} \in \{0,1\}^i$ and $j \in [2]$.

The following auxiliary proposition is crucial to prove the theorem.
\begin{proposition}
\label{property-induction}
Let $i \in [n]$. We have that $\min_{g\in \mathcal{G}_i}\eps^{(i)}(g, p^{(i)}) = Q$ if and only if for each $\bm{v} \in \{0,1\}^{i-1}$, it holds both $\p{i}{1\bv,1} \geq \p{i}{1\bv,2}$ and $\p{i}{0\bv,1} \leq \p{i}{0\bv,2}$.
\begin{proof}
Assume that for each $\bm{v} \in \{0,1\}^{i-1}$, it holds both $\p{i}{1\bv,1} \geq \p{i}{1\bv,2}$ and $\p{i}{0\bv,1} \leq \p{i}{0\bv,2}$. Then, we have that
\begin{align}
    \min_{g \in \mathcal{G}_i}\eps^{(i)}(g,p^{(i)}) &= 1 - \sum_{\bm{v} \in \{0,1\}^i}\max( \p{i}{\bv,1}, \p{i}{\bv,2}) = \nonumber \\
    &=1 - \sum_{\bm{v} \in \{0,1\}^{i-1}}\max( \p{i}{0\bv,1}, \p{i}{0\bv,2}) - \sum_{\bm{v} \in \{0,1\}^{i-1}}\max( \p{i}{1\bv,1}, \p{i}{1\bv,2}) \nonumber \\
    &= 1 - \sum_{\bm{v} \in \{0,1\}^{i-1}} \p{i}{0\bm{v},2} - \sum_{\bm{v} \in \{0,1\}^{i-1}} \p{i}{1\bm{v},1} \nonumber \\
    &= 1 - \frac{1}{2}(1 - \beta_1) + \frac{1}{2}\alpha_1 \label{equality-marginalization} \\
    &= \frac{1}{2}\left( 1 - |\alpha_1 - \beta_1| \right) = Q \nonumber
\end{align}
Equality \eqref{equality-marginalization} is simply obtained by marginalization, since $p^{(i)} \in \mathcal{P}(\bm{A}_i)$, thus $\Pr_{x \sim \mathcal{D}}( \psi_1(x) = 0 \land y(x) = 2) = (1-\beta_1)/2$ and  $\Pr_{x \sim \mathcal{D}}( \psi_1(x) = 1 \land y(x) = 1) = \alpha_1/2$.

Assume that there exists $\bm{v'} \in \{0,1\}^{i-1}$ such that $\p{i}{1\bv',1} < \p{i}{1\bv',2}$ (the case $\p{i}{0\bv',1} > \p{i}{0\bv',2}$ is proven with the same argument). Let $g^{(i)}_a$ be  defined similarly to $g_a$, i.e. $g^{(i)}_a = 1$ if $\psi_1(x) = 1$, and $g^{(i)}_a=2$ otherwise. Following the same computation of $\eqref{greedycomputation}$, we can show that $\eps^{(i)}(g_a,p^{(i)}) = Q$.
Consider the classifier $\tilde{g}$ such that $\tilde{g}(\bm{v}) = g_a(\bm{v})$ for all $\bm{v} \in \{0,1\}^i$ such that $\bm{v} \neq 1\bv'$, and $\tilde{g}(1\bm{v}')=2$.
We have that
\begin{align*}
    \eps^{(i)}(g^{(i)}_a,p^{(i)}) - \eps^{(i)}(\tilde{g},p^{(i)}) = p^{(i)}(\bm{v}', \tilde{g}(\bm{v}')) - p^{(i)}(\bm{v}', g^{(i)}_a(\bm{v})) = p^{(i)}_{1\bm{v}',2} -  p^{(i)}_{1\bm{v}',1} > 0
\end{align*}
Therefore,  $\eps^{(i)}(\tilde{g},p^{(i)}) < \eps^{(i)}(g^{(i)}_a,p^{(i)}) = Q$, which directly implies that $\min_{g\in \mathcal{G}_i}\eps^{(i)}(g, p^{(i)}) < Q$. 
\end{proof}
\end{proposition}

By induction, we will prove that for each $i \in [n]$, it is true that $\min_{g \in \mathcal{G}} \eps^{(i)}(g,p^{(i)}) = Q$.

\textbf{Base case}. Let $i=1$. We have that
\begin{align*}
    &\p{1}{1,1} = \frac{\alpha_1}{2}  & \p{1}{0,1} = \frac{1}{2}(1-\alpha_1)\\
    &\p{1}{1,2} =\frac{\beta_1}{2}  & \p{1}{0,2} = \frac{1}{2}(1-\beta_1)\\
\end{align*}
Observe that $p^{(1)} \in \mathcal{P}(\bm{A}_1)$ as the classes are balanced, and we satisfy the constraints of the matrix $\bm{A}$ for the first attribute. It is easy to observe that
\begin{align*}
    \min_{g \in \mathcal{G}} \eps^{(1)}(g,p^{(1)}) = 1 -  \frac{\alpha_1}{2} - \frac{1}{2}(1-\beta_1) = Q
\end{align*}

\textbf{Inductive step}. For $i \in 2, \ldots, n$, assume that $\min_{g \in \mathcal{G}_{i-1}} \eps^{(i-1)}(g,p^{(i-1)}) = Q$, where $p^{(i-1)}$ is solution of $\eqref{p-i-definition}$.
We will show how to construct $p^{(i)}$ from $p^{(i-1)}$ guaranteeing $\min_{g \in \mathcal{G}_i} \eps^{(i)}(g,p^{(i)}) = Q$ and that $p^{(i)} \in \mathcal{P}(\bm{A}_i)$. Observe that in that case, $p^{(i)}$ is also a solution of $\eqref{p-i-definition}$, since the classifier $g_a^{(i)}$ (defined as in the proof of \cref{property-induction}) has error exactly $Q$ with respect to any PMF $p \in \mathcal{P}(\bm{A}_i)$.

Our construction will be divided in three different cases, based on the ordering of the values $\alpha_1, \beta_1$, $\alpha_i$ and $\beta_i$. We will exhibit a different construction of $p^{(i)}$ for each of the case, but they all share the same proof line. In particular, we will guarantee that for each $\bm{v} \in \{0,1\}^{i-1}$ and $j \in [2]$, it holds
\begin{align}
\label{sum-to-previous}
    p^{(i)}_{\bm{v}1,j} + p^{(i)}_{\bm{v}0,j} =  p^{(i-1)}_{\bm{v},j}
\end{align}
This immediately implies that the classes are balanced, in fact, for any $j \in [2]$, we have that
\begin{align*}
    \sum_{\bm{v} \in \{0,1\}^i} p^{(i)}_{\bm{v},j} = \sum_{\bm{v} \in \{0,1\}^{i-1}} p^{(i)}_{\bm{v}0,j}+p^{(i)}_{\bm{v}1,j} = \sum_{\bm{v} \in \{0,1\}^{i-1}}p^{(i-1)}_{\bm{v},j} = \frac{1}{2} \enspace,
\end{align*}
where the last inequality is due to the assumption that $p^{(i-1)} \in \mathcal{P}(\bm{A}_{i-1})$. Moreover, $\eqref{sum-to-previous}$ also implies that $p^{(i)}$ satisfies the constraints imposed by the matrix $\bm{A}$ for the first $i-1$ attributes. In fact, for any $a \in [i-1]$, and $j\in[2]$, we have that
\begin{align*}
    &\sum_{\bm{v} \in \{0,1\}^i : v_{a}=1} \p{i}{\bv,j}= A_{j,a} \sum_{\bm{v} \in \{0,1\}^i} \p{i}{\bv,j} \\
    \iff &\sum_{\bm{v} \in \{0,1\}^{i-1} : v_{a}=1}\left(\p{i}{\bv0,j} + \p{i}{\bv1,j}\right) =  A_{j,a}\sum_{\bm{v} \in \{0,1\}^{i-1}}\left(\p{i}{\bv0,j} + \p{i}{\bv1,j}\right) \\
    \iff & \sum_{\bm{v} \in \{0,1\}^{i-1} : v_{a}=1} \p{i-1}{\bv,j}= A_{j,a} \sum_{\bm{v} \in \{0,1\}^{i-1}} \p{i-1}{\bv,j} \enspace .
\end{align*}
The latter equality is true as $p^{(i-1)} \in \mathcal{P}(\bm{A}_{i-1})$. 

For each different case, we will show that our construction also satisfies the constraints imposed by matrix $\bm{A}$ for attribute $i$. This, together with $\eqref{sum-to-previous}$, implies that our construction guarantees that $p^{(i)} \in \mathcal{P}(\bm{A})$.

Moreover, we will show that with our construction, we also guarantee that for each $\bm{v} \in \{0,1\}^{i-1}$, it holds that
\begin{align}
\label{inequality-to-previous}
    \p{i}{1\bv,1} \geq \p{i}{1\bv,2}\hspace{5pt} \land \hspace{5pt} \p{i}{0\bv,1} \leq \p{i}{0\bv,2} \enspace .
\end{align}
Using \cref{property-induction}, \eqref{inequality-to-previous} immediately implies that $\min_{g\in \mathcal{G}_i}\eps^{(i)}(g, p^{(i)}) = Q$. In order to show that \eqref{inequality-to-previous} holds in our construction, we will use the fact that for each $\bm{v} \in \{0,1\}^{i-2}$, it holds that
\begin{align}
\label{inequality-to-previous-i-1}
    \p{i-1}{1\bv,1} \geq \p{i-1}{1\bv,2}\hspace{5pt} \land \hspace{5pt} \p{i-1}{0\bv,1} \leq \p{i-1}{0\bv,2} \enspace .
\end{align}
This is indeed the case, as by assumption $\min_{g\in \mathcal{G}_{i-1}}\eps^{(i-1)}(g, p^{(i-1)}) = Q$, hence we can apply the other direction of  \cref{property-induction}.

We will now show our construction for the three different cases. For each case, it is straightforward to check that in our construction \eqref{sum-to-previous} holds, and that \eqref{inequality-to-previous-i-1} immediately implies \eqref{inequality-to-previous}. Therefore, we omit those computations.

\textbf{First Case}. [$\beta_1 \geq \beta_i \land \alpha_i \leq \alpha_1$].
We construct $p^{(i)}$ as follows. For each $\bm{v} \in \{0,1\}^{i-2}$,
we let
\begin{align*}
    &\p{i}{1\bv1,2} = \p{i-1}{1\bv,2} \hspace{20pt} & \p{i}{1\bv0,2} = 0 \\
    &\p{i}{1\bv1,1} = \p{i-1}{1\bv,2} + \frac{\alpha_i - \beta_1}{\alpha_1 - \beta_1}\left( \p{i-1}{1\bv,1} - \p{i-1}{1\bv,2} \right) & \\
    & \p{i}{1\bv0,1} = \frac{\alpha_1 - \alpha_i}{\alpha_1 - \beta_1}\left( \p{i-1}{1\bv,1} - \p{i-1}{1\bv,2}\right)  \enspace .
\end{align*}
These probabilities are well defined, as $0 \leq \frac{\alpha_i - \beta_1}{\alpha_1 - \beta_1} \leq 1$ and $\p{i-1}{1\bv,1} \geq \p{i-1}{1\bv,2}$. By construction, we have that $\p{i}{1\bv1,2}+\p{i}{1\bv0,2}=\p{i-1}{1\bv,2}$ and $\p{i}{1\bv1,1}+\p{i}{1\bv0,1} = \p{i-1}{1\bv,1}$, and it is easy to see that $\p{i}{1\bv1,1} \geq  \p{i}{1\bv1,2}$ and $\p{i}{1\bv0,1} \geq \p{i}{1\bv0,2}$.

For each $\bm{v} \in \{0,1\}^{i-2}$, we let
\begin{align*}
    &\p{i}{0\bv0,1} = \p{i-1}{0\bv,1} \hspace{20pt} & \p{i}{0\bv1,1} = 0 \\
    &\p{i}{0\bv0,2} = \p{i-1}{0\bv,1} + \frac{\alpha_1 - \beta_i}{\alpha_1 - \beta_1}\left( \p{i-1}{0\bv,2} - \p{i-1}{0\bv,1} \right) & \\
    & \p{i}{0\bv1,2} = \frac{\beta_i - \beta_1}{\alpha_1 - \beta_1}\left( \p{i-1}{0\bv,2} - \p{i-1}{0\bv,1}\right) 
\end{align*}
Again, by construction we have that $\p{i}{0\bv0,2}+\p{i}{0\bv1,2}=\p{i-1}{0\bv,2}$ and $\p{i}{0\bv0,1}+\p{i}{0\bv1,1} = \p{i-1}{0\bv,1}$, and it is easy to see that $\p{i}{0\bv0,2} \geq  \p{i}{0\bv0,1}$ and $\p{i}{0\bv1,2} \geq \p{i}{0\bv1,1}$.

The PMF $p^{(i)}$ satisfies the constraints imposed by the class-attribute matrix $\bm{A}$ for the attribute $i$, in fact
\begin{align*}
    \sum_{\bm{v} \in \{0,1\}^{i-1}} \p{i}{\bv1,1} = \frac{\beta_1}{2} + \frac{\alpha_i-\beta_1}{\alpha_1-\beta_1}\cdot \frac{1}{2}(\alpha_1 - \beta_1) = \frac{\alpha_i}{2} \\
        \sum_{\bm{v} \in \{0,1\}^{i-1}} \p{i}{\bv1,2} = \frac{\beta_1}{2} + \frac{\beta_i-\beta_1}{\alpha_1-\beta_1}\cdot \frac{1}{2}(\alpha_1 - \beta_1) = \frac{\beta_i}{2}
\end{align*}

\textbf{Second case}. [$\beta_1 \leq \beta_i$ $\land$ $\alpha_1 \leq \alpha_i$].
We construct $p^{(i)}$ as follows. For each $\bv \in \{0,1\}^{i-2}$, let
\begin{align*}
    &\p{i}{1\bv1,1} = \p{i-1}{1\bv,1} \hspace{20pt} & \p{i}{1\bv0,1} = 0 \\
    &\p{i}{1\bv1,2} = \p{i-1}{1\bv,2} \hspace{20pt} & \p{i}{1\bv0,2} = 0 
\end{align*}
and let
\begin{align*}
     &\p{i}{0\bv1,1} = \frac{\alpha_i-\alpha_1}{1-\alpha_1}\p{i-1}{0\bv,1} \hspace{20pt} \\
     &\p{i}{0\bv0,1} = \frac{1-\alpha_1}{1-\alpha_1}\p{i-1}{0\bv,1} \\
     &\p{i}{0\bv1,2} = \frac{\alpha_i-\alpha_1}{1-\alpha_1}\p{i-1}{0\bv,1} + \frac{(\alpha_1 - \beta_1) - (\alpha_i - \beta_i)}{\alpha_1 - \beta_1}
    \left(\p{i-1}{0\bv,2}-\p{i-1}{0\bv,1}\right) \\
    & \p{i}{0\bv0,2} =  \frac{1-\alpha_1}{1-\alpha_1}\p{i-1}{0\bv,1} +\frac{\alpha_i - \beta_i}{\alpha_1 - \beta_1}
    \left(\p{i-1}{0\bv,2}-\p{i-1}{0\bv,1}\right)
\end{align*}
By construction, we can observe that for each $\bm{v} \in \{0,1\}^{i-1}$, it holds that $\p{i}{1\bv,1} \geq \p{i}{1\bv,2}$ and that $\p{i}{0\bv,2} \geq \p{i}{0\bv,1}$. Moreover, for each $\bm{v} \in \{0,1\}^{i}$ and $j \in [2]$, it holds that $\p{i}{\bv1,j} + \p{i}{\bv0,j} = \p{i-1}{\bv,j}$. 

The PMF $p^{(i)}$ satisfies the constraints imposed by the class-attribute matrix $\bm{A}$ for the attribute $i$, in fact
\begin{align*}
        &\sum_{\bm{v} \in \{0,1\}^{i-1}} \p{i}{\bv1,1} = \frac{\alpha_1}{2} + \frac{\alpha_i-\alpha_1}{1-\alpha_1}\cdot \frac{1}{2}(1 - \alpha_1) = \frac{\alpha_i}{2} \\
        &\sum_{\bm{v} \in \{0,1\}^{i-1}} \p{i}{\bv1,2} = \frac{\beta_1}{2} + \frac{\alpha_i-\alpha_1}{1-\alpha_1}\cdot \frac{1}{2}(1 - \alpha_1) + \\
        &+\frac{(\alpha_1 - \beta_1) - (\alpha_i - \beta_i)}{\alpha_1 - \beta_1} \cdot \frac{1}{2}(\alpha_1 - \beta_1) = \frac{\beta_i}{2}
\end{align*}

\textbf{Third case}. [$\beta_i \leq \beta_1$ $\land$ $\alpha_i \leq \alpha_1$]. We construct $p^{(i)}$ as follows. For each $\bv \in \{0,1\}^{i-2}$, let
\begin{align*}
    &\p{i}{0\bv0,1} = \p{i-1}{0\bv,1} \hspace{20pt} & \p{i}{0\bv1,1} = 0 \\
    &\p{i}{0\bv0,2} = \p{i-1}{0\bv,2} \hspace{20pt} & \p{i}{0\bv1,2} = 0 
\end{align*}
and let
\begin{align*}
     &\p{i}{1\bv1,2} = \frac{\beta_i}{\beta_1}\p{i-1}{1\bv,2} \hspace{20pt} \\
     &\p{i}{1\bv0,2} = \frac{\beta_1-\beta_i}{\beta_1}\p{i-1}{1\bv,2} \\
     &\p{i}{1\bv1,1} = \frac{\beta_i}{\beta_1}\p{i-1}{1\bv,2} + \frac{\alpha_i - \beta_i}{\alpha_1 - \beta_1}
    \left(\p{i-1}{1\bv,1}-\p{i-1}{1\bv,2}\right) \\
    & \p{i}{1\bv0,1} = \frac{\beta_1-\beta_i}{\beta_1}\p{i-1}{1\bv,2} +\frac{(\alpha_1 - \beta_1) - (\alpha_i - \beta_i)}{\alpha_1 - \beta_1}
    \left(\p{i-1}{1\bv,1}-\p{i-1}{1\bv,2}\right)
\end{align*}
Again, by construction, we can observe that for each $\bm{v} \in \{0,1\}^{i-1}$, it holds that $\p{i}{1\bv,1} \geq \p{i}{1\bv,2}$ and that $\p{i}{0\bv,2} \geq \p{i}{0\bv,1}$. Moreover, for each $\bm{v} \in \{0,1\}^{i}$ and $j \in [2]$, it holds that $\p{i}{\bv1,j} + \p{i}{\bv0,j} = \p{i-1}{\bv,j}$. 

The PMF $p^{(i)}$ satisfies the constraints imposed by the class-attribute matrix $\bm{A}$ for the attribute $i$, in fact
\begin{align*}
    &\sum_{\bm{v} \in \{0,1\}^{i-1}} \p{i}{\bv1,1} = \frac{\beta_i}{\beta_1}\frac{\beta_1}{2} + \frac{\alpha_i-\beta_i}{\alpha_1-\beta_1} \cdot \frac{1}{2}(\alpha_1-\beta_1) = \frac{\alpha_i}{2} \\
        &\sum_{\bm{v} \in \{0,1\}^{i-1}} \p{i}{\bv1,2} = \frac{\beta_i}{\beta_1} \cdot \frac{1}{2} \beta_1 = \frac{\beta_i}{2}
\end{align*}

We conclude the proof by observing that since $\alpha_1 - \beta_1 \geq \alpha_i - \beta_i$, the case $[\beta_i < \beta_1 \land \alpha_1 < \alpha_i]$ is impossible. \qed

\textbf{Proof of \cref{minimax-invertible-theorem}}. \\
By combining $\eqref{Q-computation}$ and $\eqref{adversarial-quantity}$, we can rewrite $Q$ as
\begin{align*}
    Q = \max_{p \in \mathcal{P}(\bm{A})} \min_{g \in \mathcal{G}} \eps(g,p) \enspace .
\end{align*}

Consider the maximin 
\begin{align}
\label{maximin-randomized}
    Q' = \max_{p \in \mathcal{P}(\bm{A})} \min_{g_{\bm{W}} \in \mathcal{G}_R} \eps(g_{\bm{W}},p) \enspace .
\end{align}
We show that $Q = Q'$. In fact, given $p \in \mathcal{P}(\bm{A})$, it is clear that the expected error \eqref{expected-loss-randomized} of a randomized attribute-class classifier $g_{\bm{W}} \in \mathcal{G}_R$
\begin{align*}
    \eps(g_{\bm{W}}, p)  = 1-\sum_{\bm{v} \in \{0,1\}^n} \sum_{j \in [k]} W_{\bm{v},j} \cdot p(\bm{v},j)   
\end{align*}
is minimized when $\bm{W}_{\bm{v},j'}=1$ if $j' = \argmax_{j \in [k]}p(\bm{v},j)$ for all $\bm{v} \in \{0,1\}^n$, and such a attribute-class classifier is deterministic, i.e. it is equal with probability $1$ to a properly chosen classifier in $\mathcal{G}$. This proves the second part of the Theorem. 

Given $\alpha \in [0,1]$ and $p_1,p_2 \in \mathcal{P}(\bm{A})$, we define $p_\alpha = \alpha p_1 + (1-\alpha)p_2$ as a convex combination of $p_1$ and $p_2$, where for each $\bm{v} \in \{0,1\}^n$ and $j \in [k]$, we have that $p_\alpha(\bm{v},j) = \alpha p_1(\bm{v},j) + (1-\alpha)p_2(\bm{v},j)$. It is easy to verify that $p_{\alpha} \in \mathcal{P}(\bm{A})$. Moreover, for two randomized attribute-class classifiers $g_{\bm{W}}, g_{\bm{W}'}$, and $\alpha \in [0,1]$ we define $g_\alpha = g_{\alpha \bm{W} + (1-\alpha) \bm{W}'} $ as the convex combination of $g_{\bm{W}}$ and $g_{\bm{W}'}$, and observe that $g_{\alpha} \in \mathcal{G}_R$.

The sets $\mathcal{P}(\bm{A})$ and $\mathcal{W}$ are closed and bounded, therefore compact, and we have shown they are also convex. Moreover, the function $\epsilon(\cdot,\cdot)$ is bilinear with respect to $p$ and $\bm{W}$. Therefore, by von Neumann's Minimax Theorem \citep{neumann1928theorie},
the value of the minimax is equal to the value of the maximin, i.e.
\begin{align*}
     \min_{g \in \mathcal{G}_R} \max_{p \in \mathcal{P}(\bm{A})}  \eps(g_{\bm{W}},p) =  \max_{p \in \mathcal{P}(\bm{A})} \min_{g \in \mathcal{G}_R} \eps(g_{\bm{W}},p) = Q
\end{align*}
\qed

\medskip
\section{Adversarial Attribute-Class Classifier Computation}
%\alessio{Change name of this section - what would it be?}
\label{adversarial-classifier-computation}
In this section of the Appendix, we show how to compute a randomized attribute-class classifier that satisfies \cref{minimax-invertible-theorem}. First, we show that the randomized attribute-class classifier
\begin{align}
\label{minimax-w}
    g_{\bm{W}^*} = \argmin_{g_{\bm{W}} \in \mathcal{G}_R}\max_{p \in \mathcal{P}(\bm{A})} \eps(g_{\bm{W}},p)
\end{align}
satisfies the condition of the Theorem. In fact, as noted in the proof of \cref{minimax-invertible-theorem}, we have that
\begin{align*}
    \min_{g \in \mathcal{G}_R} \max_{p \in \mathcal{P}(\bm{A})} \eps(g_{\bm{W}},p) = \max_{p \in \mathcal{P}(\bm{A})} \min_{g \in \mathcal{G}_R} \eps(g_{\bm{W}},p) = \max_{p \in \mathcal{P}(\bm{A})} \min_{g \in \mathcal{G}} \eps(g,p) = Q
\end{align*}

Now, we show how to compute $\bm{W}^*$. The minimax \eqref{minimax-w} can be written as a bilinear problem. Let $w_{\bm{v},j}$ and $q_{\bm{v},j}$ be variables that denote respectively $W_{\bm{v},j}$ and $p(\bm{v},j)$ for $\bm{v} \in \{0,1\}^n$ and $j \in [k]$. Inspecting \eqref{expected-loss-randomized}, and using the fact that the minimax is equal to the maximin, we can compute \eqref{minimax-w} as
\begin{align}
    \label{maxmin-primal}
    &1 + \max_{\bm{q} \geq 0}\min_{\bm{w} \geq 0}\sum_{\bm{v} \in \{0,1\}^n} \sum_{j \in [k]} (-w_{\bm{v},j}) \cdot q_{\bm{v},j}& s.t.  \\
    &(a)  \sum_{\substack{ \bm{v} \in \{0,1\}^n :\\ v_i = 1}}q_{\bm{v},j} = A_{j,i}P_j  & \forall j \in [k], i \in [n] \nonumber  \\
    &(b)  \sum_{\substack{ \bm{v} \in \{0,1\}^n }}q_{\bm{v},j} = P_j & \forall j \in [k] \nonumber  \\
    &(c)  \sum_{j \in [k]} w_{\bm{v},j} = 1 & \forall \bm{v} \in \{0,1\}^n \nonumber 
\end{align}
We use $P_j$ to denote $P_j = \Pr_{x \sim \mathcal{D}}( y(x) = j)$ for $j \in [k]$, which is equal to $1/k$ when the classes are balanced. We transform the maximin \eqref{maxmin-primal} in a minimization problem that can be easily solved using a Linear Program.

For a given $\bm{q}$, let $\bm{w}^{\bm{q}}$ be an assignment of the variables $\bm{w}$ that  achieves the minimum. We can write the dual of the maximization problem over the variables $\bm{q}$ with respect to the fixed $\bm{w}^{\bm{q}}$ as
\begin{align*}
    &1 + \min_{\substack{\bm{a} \in \mathbb{R}^{k} \\ \bm{b} \in \mathbb{R}^{k \times n} }} \left( \sum_{j \in [k]} P_j \cdot a_j + \sum_{j \in [k]}\sum_{i \in [n]} P_j \cdot M_{j,i} \cdot b_{j,i}\right)& s.t.  \\
    &(a)  \hspace{5pt}  a_{j} + \sum_{\substack{i \in [n] \\ v_i = 1}}b_{j,i} \geq -w^{\bm{q}}_{\bm{v},j}   & \forall \bm{v} \in \{0,1\}^n , j \in [k] \nonumber  
\end{align*}
Due to strong-duality, the optimal value of the dual problem is the same of the primal with respect to the fixed assignment $\bm{w}_q$. By choosing $\bm{w}_q$ as the minimum over all feasible $\bm{w}$, we finally obtain the following minimum problem whose optimal value is equal to \eqref{maxmin-primal}.
\begin{align}
\label{minimini}
    &1 + \min_{\substack{\bm{a} \in \mathbb{R}^{k} \\ \bm{b} \in \mathbb{R}^{k \times n} \\ \bm{w} \geq 0}} \left(\sum_{j \in [k]} P_j \cdot a_j + \sum_{j \in [k]}\sum_{i \in [n]} P_j \cdot M_{j,i} \cdot b_{j,i}\right)& s.t.  \\
    &(a)  \hspace{5pt}a_{j} + \sum_{\substack{i \in [n] \\ v_i = 1}}b_{j,i} \geq -w_{\bm{v},j}   & \forall \bm{v} \in \{0,1\}^n , j \in [k] \nonumber  \\
    &(b) \sum_{j \in [k]} w_{\bm{v},j} = 1 & \forall \bm{v} \in \{0,1\}^n \nonumber 
\end{align}

This minimization problem is easily solved as a Linear Programming with $O(k \cdot 2^n)$ variables and constraints. We choose $\bm{W}^*$ as the optimal solution $\bm{w}^*$ of the minimum \eqref{minimini}.

\section{Approximation of the Lower Bound}
\label{app:approximation-subsection}
In this section of the Appendix, we show a computationally efficient method to compute a lower bound to the value $Q$ in a multiclass classification setting, i.e. $k \geq 2$. We build on the results of Section~\ref{binary-exact-computation}, and we will approximate $Q$ by using~\Cref{binary-adversarial-computation} between properly chosen pairs of the $k$ classes.
Consider a weighted, undirected complete graph $G$. Each vertex of the graph represents a class $j \in [k]$, and the edge $\{j,j'\}$ between classes $j,j' \in [k]$, $j \neq j'$, has weight $w_{\{j,j'\}} = \frac{1}{2}\left( 1 - \max_{i \in [n]}|A_{ji} - A_{j'i}|\right)$ computed as in~\Cref{binary-adversarial-computation}. A matching $M$ is a subset of edges such that no two  edges of $M$ share an endpoint, i.e. for each  $e,e' \in M$, $e \neq e'$, we have that $e \cap e' = \emptyset$. The weight of a matching $M$ is defined as the sum $\sum_{e \in M}w_e$ of the weights of the edges of $M$. The following theorem relates the weight of a matching to the value $Q$.

\begin{theorem}
\label{approximate-Q-matching}
Let $M$ be a matching of $G$, and $Q$ be computed as in \eqref{adversarial-quantity}. Then,
    $Q \geq \frac{2}{k} \sum_{e \in M} w_e$ .
\end{theorem}

\begin{proof}
For each edge $\{j,j'\}= e \in M$, consider the matrix $\bm{A}^e$ obtained by selecting the two rows of the classes $j$ and $j'$. Let $p^e$ be the PMF over $\{0,1\}^n \times \{j,j'\}$ that achieves the maximum of \cref{binary-adversarial-computation}. That is, $p^e$ is the adversarial distribution if we were to only distinguish between the two balanced classes $j$ and $j'$ assuming that we need to also satisfy the constraints imposed by $\bm{A}^e$.

Let $\bm{C} = [k] \setminus \left( \cup_{e \in M} M \right)$ be the set of classes that do not belong to any edge of the matching $M$. For any $c \in \mathcal{C}$, we let $p^c$ be an arbitrary PMF over $\{0,1\}^n \times \{c\}$ that satisfies the constraints imposed by the $c$ row of the matrix $\bm{A}$. We give a simple example of such a PMF $p^c$, assuming independence between the attributes. For each $\bm{v} \in \{0,1\}^n$, we let
\begin{align*}
    p^c(\bm{v},c) = \prod_{i \in [n]} A_{c,i}^{v_i} \prod_{i \in [n]} (1-A_{c,i})^{1-v_i} \enspace ,
\end{align*}
and it is easy to verify that this PMF satisfies the constraints imposed by the row $c$ of matrix $\bm{A}$.

Based on the previous PMFs, we define a PMF $\tilde{p} \in \mathcal{P}(\bm{A})$ over $\{0,1\}^n \times [k]$. For each $\bm{v} \in \{0,1\}^n$ and $j \in [k]$, we let
\begin{align*}
    \tilde{p}(\bm{v},j) = \begin{cases}
        \frac{1}{k}p^j(\bm{v},j) \hspace{10pt} \mbox{if } j \in C \\
        \frac{2}{k}p^e(\bm{v},j) \hspace{10pt} \mbox{for } e \in M : j \in e
    \end{cases}
\end{align*}
Observe that this PMF is well defined, as each class is either in $C$ or it belongs to a unique edge in the matching $M$. Moreover, by construction of $\tilde{p}$, the classes are balanced and they satisfy the constraints imposed by matrix $\bm{A}$. 

For each $\{j,j'\} = e \in M$, by construction we have that $1 - \sum_{\bv}\max(p^e(\bv, j),p^e(\bv, j') ) = \sum_{\bv}\min(p^e(\bv, j),p^e(\bv, j') ) = w_e$ (as $p_e$ achieves the maximum of \cref{binary-adversarial-computation}). 
We have that:
\begin{align*}
    Q \geq \min_{g \in \mathcal{G}} \eps(g, \tilde{p}) &= 1 - \sum_{\bm{v}} \max_{j \in [k]} \tilde{p}(\bm{v},j) \\
    &\geq \sum_{\{j,j'\} = e \in M}\sum_{\bm{v}}  \min\left( \tilde{p}(\bm{v},j), \tilde{p}(\bm{v},j')\right) \\
    & = \frac{2}{k} \sum_{\{j,j'\} = e \in M}\sum_{\bm{v}}  \min\left( p^e(\bm{v},j), p^e(\bm{v},j')\right) \\
    & = \frac{2}{k} \sum_{e \in M} w_e
\end{align*}
The first inequality is true because $M$ is a matching, so no two edges of $M$ share an endpoint, and the second equality is due to the definition of $\tilde{p}$.
\end{proof}

In order to maximize the lower bound provided by \cref{approximate-Q-matching}, we want to find a matching of $G$ with maximum weight. This optimization problem can be solved in $O(k^3)$ time by using an optimized version of the blossom algorithm~\citep{edmonds1965paths,lawler2001combinatorial}.
%\eli{can we relate to the optimal?}

\section{Experimental details}\label{app:exp}

In this section, we provide additional details for the experiments in~\Cref{sec:experiments}. 

\subsection{Data}\label{app:data}

We choose the following four datasets with attributes that are widely used benchmarks in ZSL. 

\textbf{Animals with Attributes 2} (AwA2) consists of 37,322 images of 50 animal classes that are split into 40 seen and 10 unseen classes~\citep{xian2018zero}.
%The dataset has two class-attribute matrices: one is binary, and one continuous.
The dataset contains 85 attributes.
We normalize the provided continuous-valued class-attribute matrix, whose entries indicate the strength of the class-attribute association, which we interpret as a probability. We use this matrix as class-attribute matrix.
We use the provided binary class-attribute matrix to infer image-level attribute representation for each image to learn the attribute detectors (\Cref{app:attr_detectors}).

\textbf{aPascal-aYahoo} (aPY) consists of 15,339 images of 32 classes of animals and means of transportation, that are split into 20 seen and 12 unseen classes~\citep{farhadi2009describing}. 
Each image is annotated with 64 attributes. %Following the standard~\citemissing, we average the attribute representation of the images of each class to obtain the class-attribute matrix for both seen and unseen classes.

\textbf{Caltech-UCSD Birds-200-2011} (CUB) consists of 11,788 images of 200 fine-grained birds classes that are split into 150 seen and 50 unseen classes~\citep{WahCUB_200_2011}. 
Each image is annotated with 312 attributes.

\textbf{SUN attribute database} (SUN) consists of 14,340 images of 717 scenes, e.g., ballroom and auditorium, that are split into 645 seen and 72 unseen classes~\citep{patterson2014sun}. 
Each image is annotated with 102 attributes. 

For each image, both the SUN and CUB datasets provide multiple crowdsourced attribute annotations. 
We average such annotations, and we obtain a  continuous attribute-representation of each images.
For our purposes, i.e., the training of the attribute detectors (\Cref{app:attr_detectors}), we round the value of each attribute.

%\textbf{Large-scale Dataset for ZSL} (LAD)~\citemissing consists of 78,017 images of 230 classes, e.g., fruits, vehicles, and Electronics, that are split into 184 seen and 46 unseen classes.
% Each image is annotated with 359 attributes. 

For all these datasets, we use the split between seen and unseen classes suggested by~\citet{xian2018zero}.
Except for AwA2, we obtain the class-attribute matrices by averaging the attribute representation of the images of each class. This is the same strategy used by~\citet{RomeraParedes2015AnES} in their experiments.
For each dataset, we use a pre-trained ResNet-101~\citep{He2016DeepRL} as an encoder to extract features from the images.
The features are $2048$-dimensional, and they are used as input for the ZSL models.

\textbf{Synthetic Data Generation}. In Section~\ref{sec:exp:adversarial}, we generate adversarial synthetic data based on a input class-attribute matrix $\bm{A} \in [0,1]^{k \times n}$. To this end, we compute the solution of the Linear Program presented in \Cref{sec:compute_bound}. The values of the variables $q_{\bm{v},j}$ that achieve the minimum value of the Linear Program denote an adversarial distribution over the attributes and (balanced) classes that satisfy the constraints imposed by the class-attribute matrix.
We remind that $q_{\bm{v},j}$ denotes the probability that an image has attribute representation $\bm{v}$ and it belongs to class $j$.
We sample classification items from this distribution as follows. We sample a class uniformly at random among the $[k]$ classes, and then we sample a feature vector $\bm{x}$ with probability $k \cdot q_{\bm{x},j}$ with $\bm{x} \in \{0,1\}^n$ (That is, the feature representation is equal to the attribute representation). It is clear that data sampled in this way satisfy the constraints imposed by the class-attribute matrix with attribute functions $\psi_i(\bm{x}) = x_i$ for $i \in [n]$.

\subsection{Learning Attribute Detectors}\label{app:attr_detectors}
For DAP, we need to learn an attribute detector for each attribute. 
The attribute detectors are classifiers that given an image output either $1$ or $0$, if the attribute appears in the image or not, respectively.
We learn the attribute detectors in a supervised fashion on the seen classes, by using the attribute annotations of the images.
In AwA2, the attributes are not explicitly annotated for each image, and we use the discrete attribute description of the image's class. 

\subsection{ZSL models and training details}\label{appendix:zsl}
In our experiments, we compare the lower bound on the error with multiple ZSL methods. 
Here, we provide details about the methods how we train them.

\textbf{DAP}~\citep{Lampert2014AttributeBasedCF}. This method is the first attribute-based method to solve the ZSL problem in the visual domain.
It uses attribute detectors trained on the seen classes, and then uses the class-attribute matrix to infer the a posteriori most-probable unseen class.
DAP unrealistically assumes attribute independence. 
We train attribute detectors as explained in the previous section. As suggested by~\citet{Lampert2014AttributeBasedCF}, we use a uniform prior on the unseen classes. 
The implementation of DAP is based on the code released by~\citet{Lampert2014AttributeBasedCF}\footnote{\url{https://github.com/zhanxyz/Animals_with_Attributes}} under the MIT License\footnote{\url{https://opensource.org/licenses/MIT}}. 

ESZSL~\citep{RomeraParedes2015AnES}, SAE~\citep{Kodirov2017cvpr}, ALE~\citep{Akata2016}, and SJE~\citep{Akata2015cvpr} learn bilinear maps from image features to the the rows of the class-attribute matrix. At training, they use the class-attribute matrix of the seen classes, while for predictions they use the one of the unseen classes. These methods differ in the definition of the learning objective and the optimization method. 
In particular, ESZSL and SAE have closed form solutions.

\textbf{ESZSL.} The hyperparameters of the model are $\alpha$ and $\gamma$, which are the regularizer parameter for feature space and the regularizer parameter for the attributes space, respectively. The parameters $\alpha$ and $\gamma$ for each dataset are set as follows: aPY, $\alpha=3$ and $\gamma=-1$;  AwA2, $\alpha=3$ and $\gamma=0$; CUB, $\alpha=3$ and $\gamma=-1$; SUN, $\alpha=3$ and $\gamma=2$.

\textbf{SAE.} The hyperparameter of the model is $\lambda$ which is a coefficient that controls the trade-off between the decoder and encoder losses. The values of $\lambda$ are set as follows for each dataset: aPY, $\lambda=4$; AwA2, $\lambda=0.2$; CUB, $\lambda=0.2$; SUN, $\lambda=0.16$.

\textbf{ALE.} The hyperparameters of the models are the normalization strategy applied to the class-attribute matrix, and the SGD learning rate $\gamma$. For each dataset, the normalization strategy and the learning rates are: aPY, $\ell_2$ and $0.04$; AwA2, $\ell_2$ and $\gamma=0.01$; CUB: $\ell_2$ and $\gamma=0.3$;  SUN, $\ell_2$ and $\gamma=0.1$.

\textbf{SJE.} The hyperparameters of the model are the normalization strategy applied to the class-attribute matrix, the SGD learning rate $\gamma$, and the margin $m$ for the optimization of the objective. For each dataset we report these parameter, in order: aPY, no normalization, $\gamma=0.01$, and $m=1.5$; AwA2, $\ell_2$, $\gamma=1$, and $m=2.5$; CUB, mean-centering, $\gamma=0.1$, and $m=4$; SUN, mean-centering, $\gamma=1$, and $m=2$.

The chosen hyperparameters maximize the balanced accuracy on the validation classes, and lead to test errors on the unseen classes comparable with the benchmarks by~\citet{xian2018zero}.
The implementations of ESZSL, SAE, ALE, SJE are based on a public code repository\footnote{\url{https://github.com/mvp18/Popular-ZSL-Algorithms}} released under the MIT License.
In~\Cref{tab:acc:zsl_acc}, we report the balanced accuracy of each method when trained on the whole set of attributes and seen classes.
Interestingly, we note that the accuracy of the methods on aPY is comparable to the accuracy that the methods achieve by only using $15$ attributes (\Cref{fig:err_attr}).

The last ZSL method we consider is DAZLE~\citep{Huynh-DAZLE:CVPR20} which (1) uses dense attribute-based attention to find local discriminative regions, and (2) embeds each attribute-based feature with the attribute semantic description. The implementation of DAZLE is based on the code released by~\citet{Huynh-DAZLE:CVPR20}\footnote{\url{https://github.com/hbdat/cvpr20_DAZLE}} under the MIT License. 

\textbf{DAZLE.} The hyperparameters of the model the weight of the the self-calibration loss $\lambda$, the learning rate $\gamma$, the weight decay $w$, and momentum $m$. For each dataset we report these parameter, in order: aPY, $\lambda=0.1$, $\gamma=0.0001$, $w=0.0001$, and $m=0$; AwA2, $\lambda=0.1$, $\gamma=0.0001$, $w=0.0001$, and $m=0$; CUB, $\lambda=0.1$, $\gamma=0.0001$, $w=0.0001$, and $m=0.9$; SUN,$\lambda=0.1$, $\gamma=0.0001$, $w=0.0001$, and $m=0.9$. We used the same setting as in the released implementation of the model.

\textbf{Resources.} We run the experiments on an internal cluster. Most of the methods are executed on CPUs, while for DAZLE we used a GPU NVIDIA GeForce RTX 3090.

\begin{table}[t]
    \centering
    \resizebox{.7\linewidth}{!}{
    \begin{tabular}{lcccc}
    \toprule
    Method &  \multicolumn{1}{c}{aPY} & \multicolumn{1}{c}{AwA2} &  \multicolumn{1}{c}{CUB} & \multicolumn{1}{c}{SUN} \\
    \midrule 
    DAP~\citep{Lampert2014AttributeBasedCF} &  30.32 & 40.44 & 27.99 & 19.65   \\
    ESZSL~\citep{RomeraParedes2015AnES} &  38.56 & 54.82  & 53.95 & 55.69  \\
    SAE~\citep{Kodirov2017cvpr} &  16.49 & 58.89  &  46.71  & 59.86 \\
    %DeVISE~\cite{Frome2013DeViSEAD} & 60.25 $\pm$ 2.63  & 33.40 $\pm$ 2.13 & 46.80 $\pm$ 0.94 & 54.74 $\pm$ 1.08 \\
    ALE~\cite{Akata2016} &  33.52 $\pm$ 0.35 & 52.78 $\pm$ 2.78 & 51.38 $\pm$ 0.77 & 61.69 $\pm$ 0.40\\
    SJE~\citep{Akata2015cvpr} &  31.93 $\pm$ 0.41 & 69.17 $\pm$ 1.89 & 52.23 $\pm$ 0.19 & 52.94 $\pm$ 0.70\\
    DAZLE~\citep{Huynh-DAZLE:CVPR20} &  31.46 $\pm$ 1.52 & 67.57 $\pm$ 1.33 & 57.23 $\pm$ 0.70 & 56.15 $\pm$ 0.57 \\
    %\midrule
    %ADVP-15 & 40.44 & 25.39 & - & - & - \\
    %ADVT-15 & 90.00 & 79.08 & - & - & - \\
    %Approx. ADVT-all & 97.33 & 94.59 & 94.56 & 96.92 & - \\
    \bottomrule
    \end{tabular}}
    \caption{We report the Top-1 balanced average unseen class-accuracy, and standard errors over 5 seeds,  for popular attribute-based ZSL. All the metrics are obtained using the splits proposed in~\citet{xian2018zero}. ESZSL, SAE and DAP do not have intervals because have a closed form solution. For SAE we report results from the semantic to the feature space (\citet{Kodirov2017cvpr}, Section 4.1).}
\label{tab:acc:zsl_acc}
\end{table}

% \iffalse
% \section{Extension of~\Cref{sec:exp:adversarial}}

% In this section, we extend the experiments we run in~\Cref{sec:exp:adversarial} by investigating the behaviour of the ZSL models error rates and the lower bound on the validation classes.
% The validation classes are a set of seen classes that are held-out during training, and they are used to choose the ZSL models' hyperparameters. 
% However, these validation classes are not directly used to learn the map from image to attributes, and we would expect their error rate to compare similarly to the lower bound, as observed on the unseen classes.
% In~\Cref{fig:err_attr_val}, for all the datasets but AwA2, we observe that the gap between the error rates of the ZSL models and the lower bound is smaller than the one observed in~\cref{fig:err_attr} for the unseen classes. We speculate that this happens because the models are able to fine tune their hyperparameters with respect to those validation classes, hence the image from attribute map is able to generalize better with respect to those classes. In fact, DAP and ATA perform consistently worse than the other models on the validation classes, as they do not perform any fine tuning on the validation classes.
%  \fi

\newpage
\section{Additional Experimental Results}
\label{app:exp:res}

\subsection{Additional Results for \Cref{sec:exp:adversarial}}
In \Cref{fig:err_attr_SUN}, we report the results for the experiments of \Cref{sec:exp:adversarial} for the SUN dataset. The experiments are consistent with our findings. We observe that for the experiments on the SUN data, we still observe that the empirical error of ZSL models roughly follows the trend of the lower bound. This suggests that the lower bound is able to capture how the additional information provided by an attribute leads to improvements of the ZSL models. 
Moreover, for the adversarially generated synthetic data, we observe that no method is able to achieve errors lower than the lower bound, consistently with our theory.

In this subsection, we also report the empirical results for a method that we call \textbf{APA} (Adversarial Predicted Attributes) across all four datasets.
APA is an adversarial algorithm that uses a map from attributes to classes that satisfies~\cref{minimax-invertible-theorem}, computed as in~\cref{adversarial-classifier-computation}.  The method uses attribute detectors trained on the seen classes (\cref{app:attr_detectors}), and predicts the target classes according to the output of those detectors, as specified in \cref{subsection:tight}.
APA is similar to DAP, except in how the attribute detectors are used to predict the target classes. In~\Cref{fig:apa}, we report the result of the experiments for the method APA. We point out that the performance of APA is competitive to the other ZSL approaches, at least using a reduced number of attributes. We remark that this comparison is out of the scope of this paper, but we believe these results open promising direction for further development of adversarial attribute-based zero-shot learning models.

\begin{figure*}[t]
    \centering
    \includegraphics[width=.4\columnwidth]{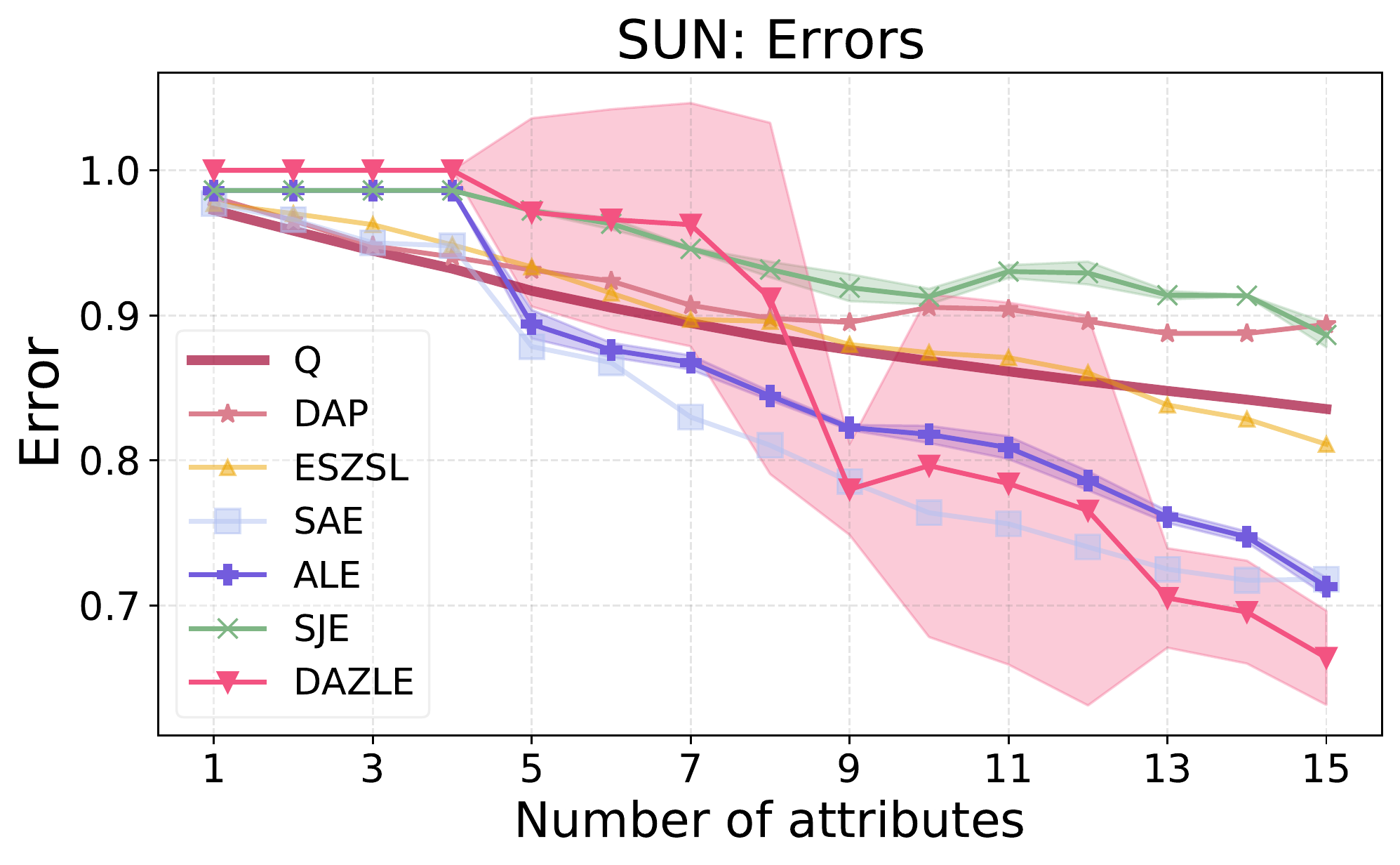}
    \includegraphics[width=.4\columnwidth]{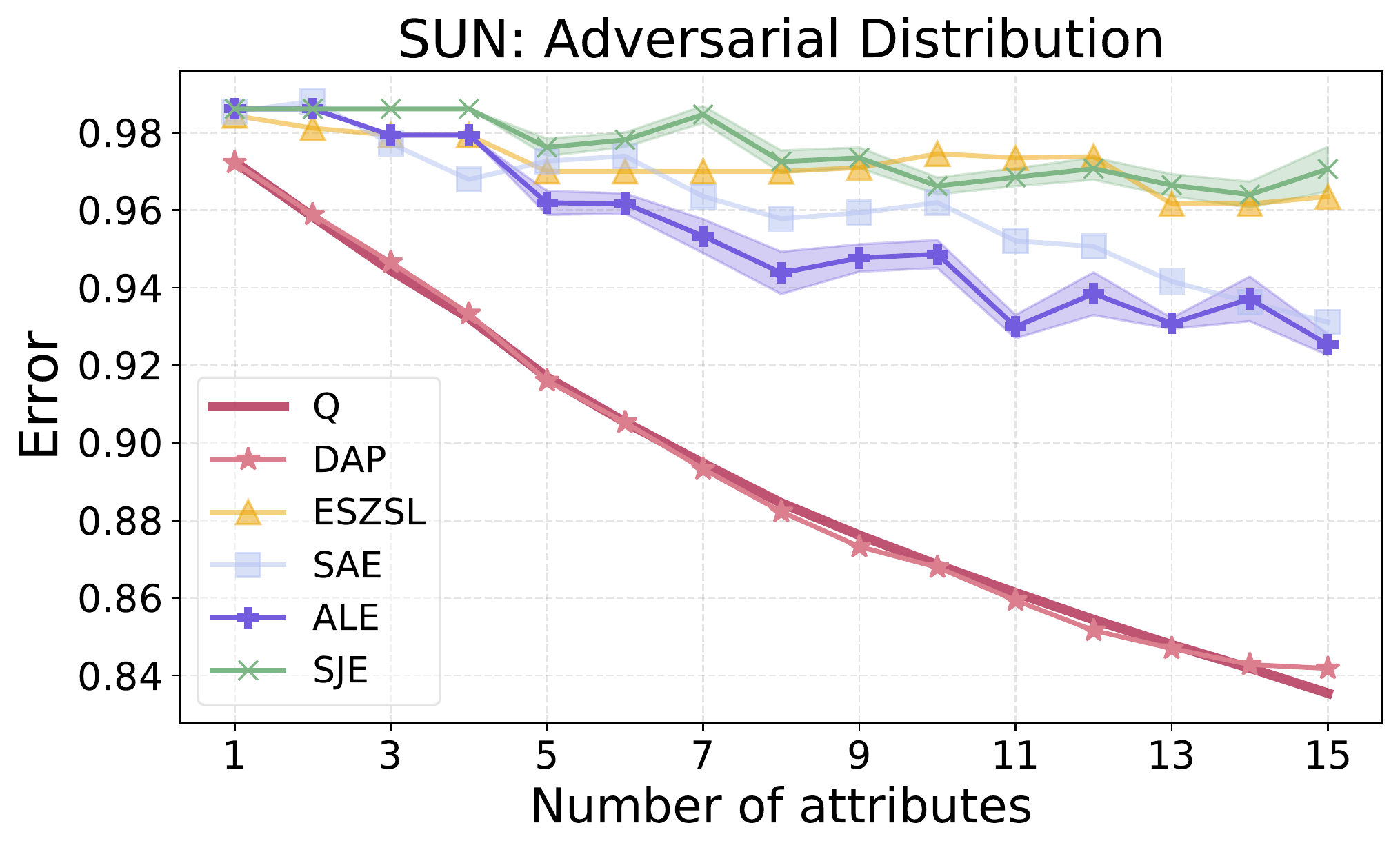}
    \caption{\textbf{Comparison of the lower bound with the empirical error.} We plot the lower bound on the error (\textbf{Q}), and the error of ZSL methods with attributes (\textbf{DAP, ESZSL, SAE, ALE}, and \textbf{DAZLE}).
    The first column reports these values computed on the unseen classes of SUN dataset, varying the number of available attributes. %The ZSL error can be much higher than $Q$ most probably because of the error introduced by the domain-shift between seen
    %and unseen classes.
    The second column reports the values for the adversarially generated synthetic data. 
    %~\cref{sec:exp:adversarial} provides detailed analysis of the results. 
    The bands indicate the standard errors on five runs with different seeds for randomized methods.
    %We defer to~\cref{app:exp} the plots of SUN.
    %These results validate that even in the absence of domain shift, there exists a distribution of the data that satisfy the constraints imposed by the class-attribute matrix for which no method can do better than the lower bound.
    }
    \label{fig:err_attr_SUN}
\end{figure*}

\subsection{Additional Results for \Cref{sec:exp:misclassification}}
In this subsection, we extend the experiments of \Cref{sec:exp:misclassification} and we propose a way to analytically quantify the similarity between the pairwise lower bound error matrix $L$ and the misclassification matrices.
Suppose that a model makes $m$ errors $E = \{ (j_1,j'_1), \ldots, (j_m, j'_m) \}$ on the data, where for each $i \in [m]$, the pair $(j_i,j'_i)$ represents an instance where the ZSL model outputs $j_i$ but the true class is $j'_i \neq j_i$. We compute the ratio between the empirical expected weight of the errors $E$ according to the graph $G$ and the expected weight of errors made uniformly at random between the classes, i.e.
    $\left(\frac{1}{m}\sum_{(j,j') \in E} w_{j,j'}\right)/\left(\frac{1}{k(k-1)} \sum_{j \neq j'} w_{j,j'}\right)$.
We name this quantity \emph{skeweness} (Sk),  and we observe that if the ratio is greater than $1$, then the misclassification errors $E$ of the ZSL models are skewed towards pair of classes that have larger values in $\bm{L}$, i.e., they are hard to distinguish.
In~\cref{tab:metrics}, we report the skewness scores computed for all the combination of ZSL models and datasets. We observe that all these quantities are greater than $1$. As noted before, this shows that the errors are skewed towards those indicated by our theoretical analysis. We can observe that the skewness is approximately $1$ for SAE on the aPY dataset. This is not surprising, as the model has very low performance (16.49\% accuracy, see~\cref{tab:acc:zsl_acc} in \cref{app:exp}) on this ZSL task.
We point out that it is very challenging to define a pairwise metric between the entries of $\bm{L}$ and the misclassification matrix $\bm{M}$ to describe their similarity. A pairwise metric would fail to capture more complex relations between classes. For instance, consider the 
scenario where  
three classes
are very hard to distinguish according to the values of their lower bound in $\bm{L}$. A ZSL model could fail to distinguish between them, and it could always output the same class given an image of any of those three classes. This would imply that we will observe zero misclassifications between a pair of these two classes in the matrix $\bm{M}$, which is different from the same entry in $\bm{L}$. Contrarily, the skewness metric is not affected by this problem.

%\begin{wraptable}{r}{7.0cm}
\begin{table}[t]
    \centering
    \resizebox{0.5\columnwidth}{!}{
    \begin{tabular}{lcccc}
    \toprule
    Method &  \multicolumn{1}{c}{aPY} & \multicolumn{1}{c}{AwA2} & \multicolumn{1}{c}{CUB} & \multicolumn{1}{c}{SUN}\\
    \midrule
    % &  \textbf{Sk} &  \textbf{Sk} &  \textbf{Sk}  & \textbf{Sk}  \\
    % \midrule
    DAP &  3.65 $\pm$ 0.04 & 3.44 $\pm$ 0.04 &  2.96 $\pm$ 0.01 &  3.09 $\pm$ 0.00\\
    ESZSL &  3.94 $\pm$ 0.07 & 3.11 $\pm$ 0.03 & 3.38 $\pm$ 0.02 &  3.45 $\pm$ 0.00 \\
    SAE &  1.20 $\pm$ 0.01 &  3.33 $\pm$ 0.02 & 3.65 $\pm$ 0.02 &   3.55 $\pm$ 0.00 \\
    %DeVISE & 0.56  & 0.35 & 4.11 & 0.71  & 0.26 & 2.26  & 0.74 & 0.29 & 3.55 & 0.87 & 0.26 & 3.60 \\
    SJE &   5.09 $\pm$ 0.04 &   3.37 $\pm$ 0.03 &  3.32 $\pm$ 0.02 &   3.44 $\pm$ 0.00 \\
    ALE &  4.89 $\pm$ 0.04 &  3.56 $\pm$ 0.02 &  3.68 $\pm$ 0.01 & 3.59 $\pm$ 0.00 \\
    DAZLE &  3.93 $\pm$ 0.05 &  3.07 $\pm$ 0.05 &  3.87 $\pm$ 0.01 & 3.51 $\pm$ 0.00 \\
    \bottomrule
    \end{tabular}}
    \caption{We report for each  model the average \emph{skewness} and standard deviation over class-balanced test sets. A value greater than $1$ indicates that the model's misclassification error is skewed toward pairs of classes with large values in $\bm{L}$. %The standard deviation on SUN is zero because classes are balanced and the samples are all the same.
    }
\label{tab:metrics}
%\vspace{-1.5em}
\end{table}
%\end{wraptable} 

% \begin{table}[t]
%     \centering
%     \resizebox{\columnwidth}{!}{
%     \begin{tabular}{lcccccccccccc}
%     \toprule
%     Method &  \multicolumn{3}{c}{aPY} & \multicolumn{3}{c}{AwA2} & \multicolumn{3}{c}{CUB} & \multicolumn{3}{c}{SUN}\\
%     \midrule
%     & \textbf{MtoG} & \textbf{GtoM} & \textbf{Sk} & \textbf{MtoG} & \textbf{GtoM} & \textbf{Sk} & \textbf{MtoG} & \textbf{GtoM} & \textbf{Sk} & \textbf{MtoG} & \textbf{GtoM} & \textbf{Sk}  \\
%     \cmidrule(lr){2-4} \cmidrule(lr){5-7} \cmidrule(lr){8-10} \cmidrule(lr){11-13}
%     DAP & 0.43 & 0.66 &  3.60 & 0.71 & 0.22  & 3.41 & 0.64 & 0.31 &  2.96 & 0.80 & 0.29 & 3.09\\
%     ESZSL & 0.40 & 0.53 &  3.93 & 0.66 & 0.26 &  3.12 & 0.70 & 0.25 & 3.38 & 0.86  & 0.27 & 3.45 \\
%     SAE & 0.11  & 0.13 & 1.20 & 0.78 & 0.30 & 3.31 & 0.77 & 0.28 & 3.66 &  0.91 & 0.25 &  3.55 \\
%     %DeVISE & 0.56  & 0.35 & 4.11 & 0.71  & 0.26 & 2.26  & 0.74 & 0.29 & 3.55 & 0.87 & 0.26 & 3.60 \\
%     SJE &  0.33 & 0.40 & 5.15 &  0.75 & 0.26 & 3.40  & 0.67 & 0.26 & 3.31 & 0.89 & 0.27 &  3.44  \\
%     ALE & 0.60 &  0.60 & 4.81  & 0.83 & 0.22 & 3.55  & 0.76 & 0.26 & 3.66 & 0.89 & 0.23 &  3.59 \\
%     \bottomrule
%     \end{tabular}}
%     \caption{\Cri{To do}}
% \label{tab:metrics}
% \end{table} 
\textbf{Additonal Details.} As our lower bound is computed assuming balanced classes, we ensure this assumption holds by sampling the test data uniformly among the unseen classes.
In~\cref{tab:metrics}, we report the skewness averaged on 10 different randomly selected subsets of test data, and the respective standard deviation.
Specifically, for each class we sample a number of images equal to the minimum class-size among the unseen classes.

\begin{figure*}[t]
    \centering
    \includegraphics[width=.45\linewidth]{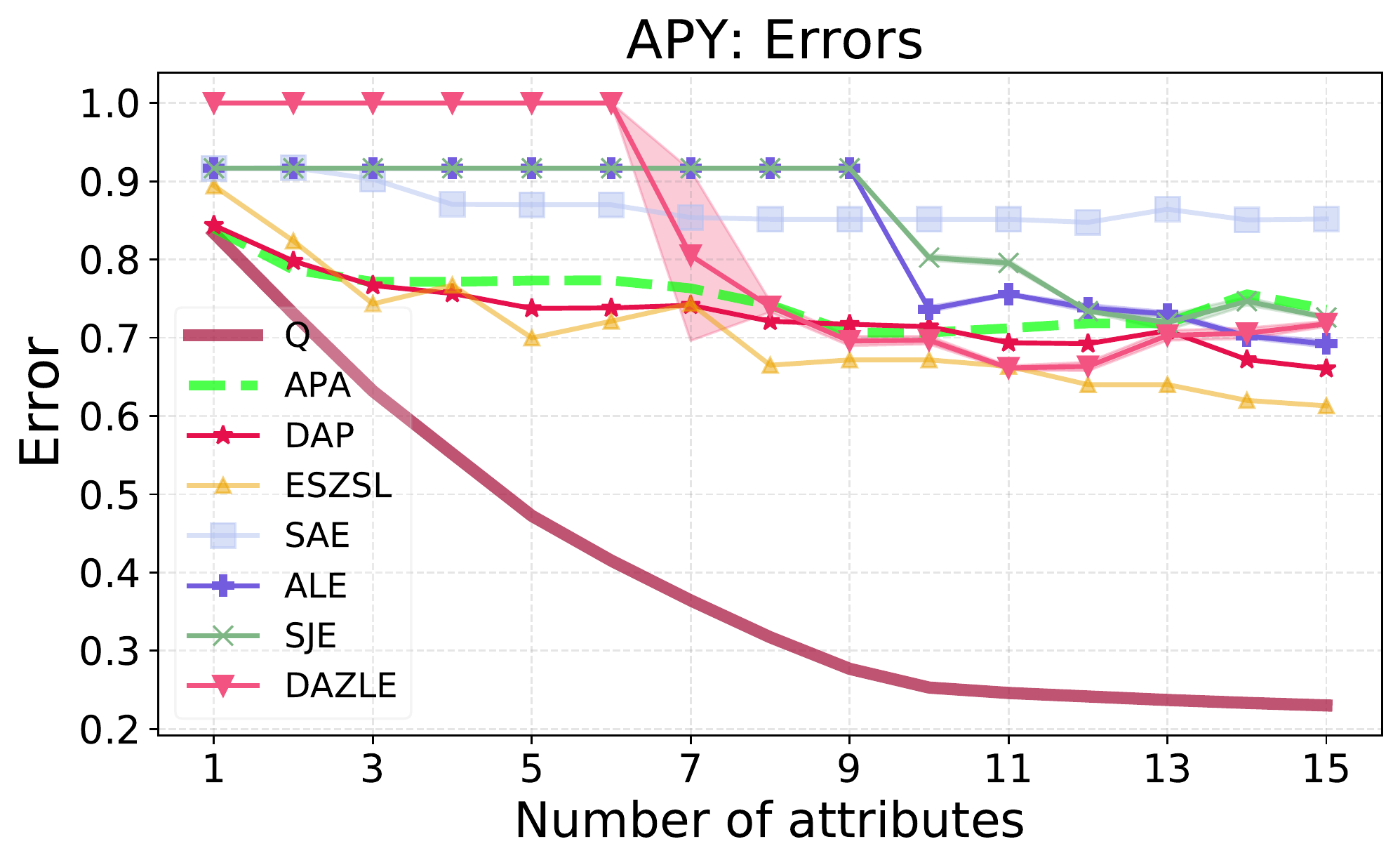}
    \includegraphics[width=.45\linewidth]{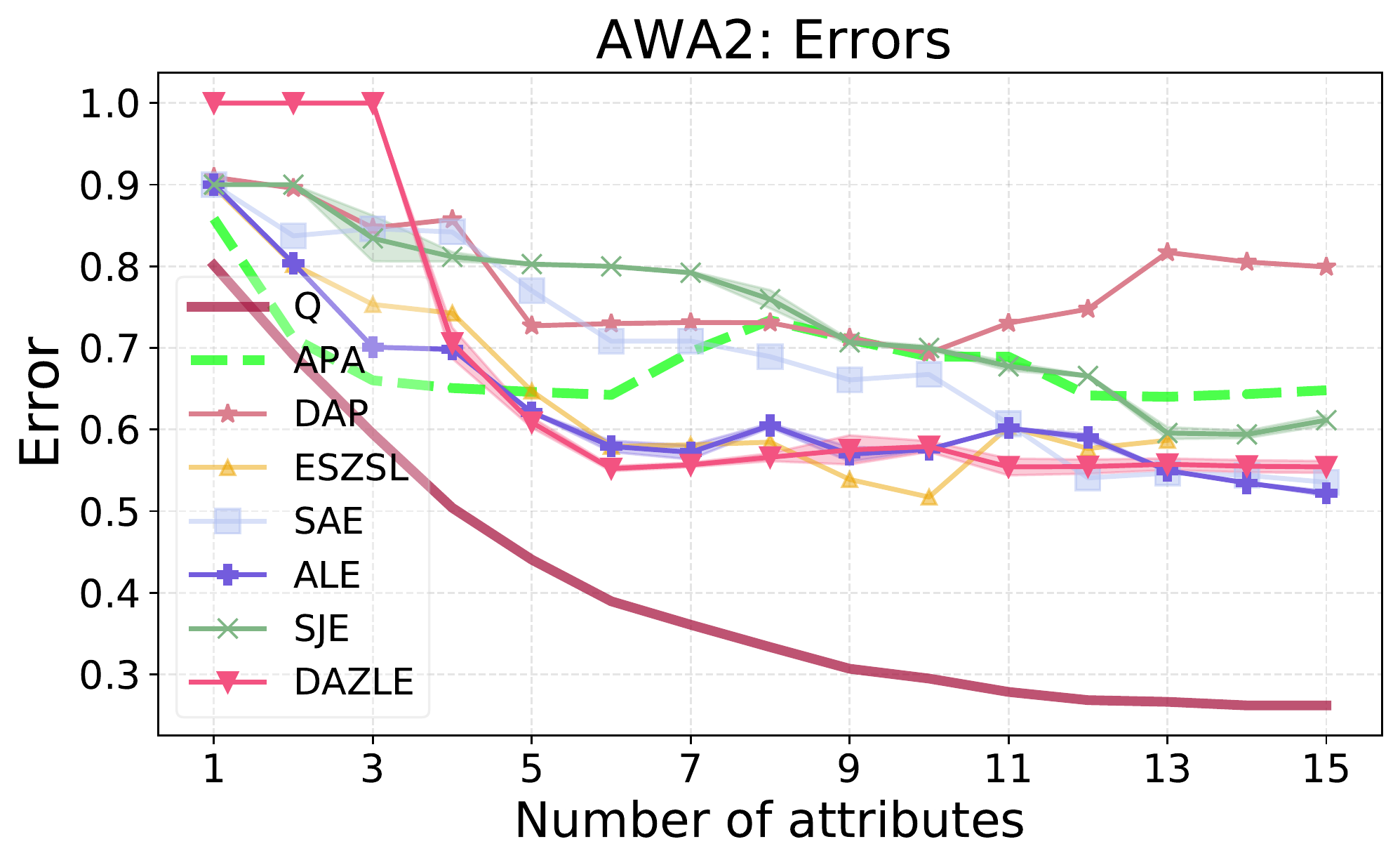}
    \includegraphics[width=.45\linewidth]{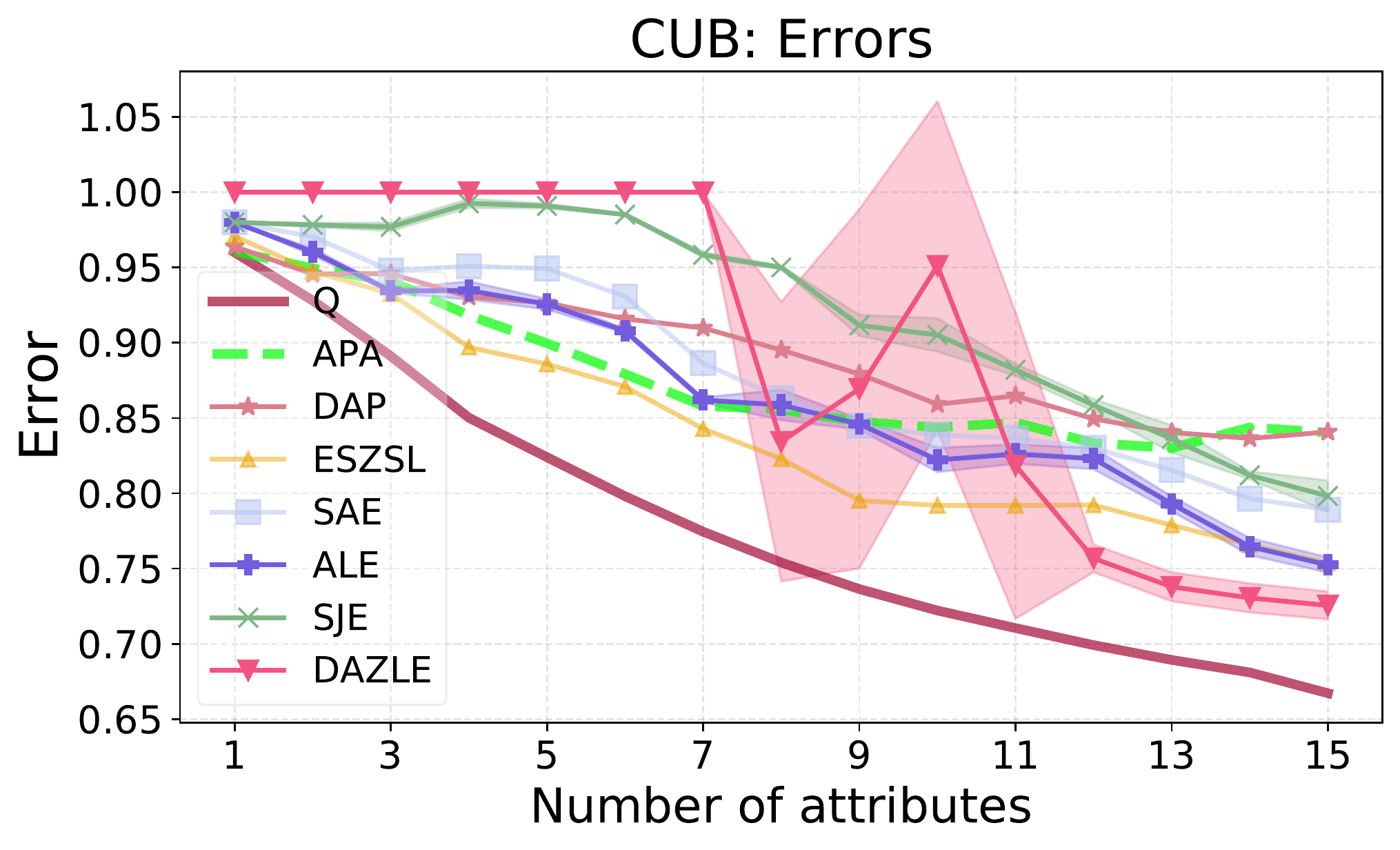}
    \includegraphics[width=.45\linewidth]{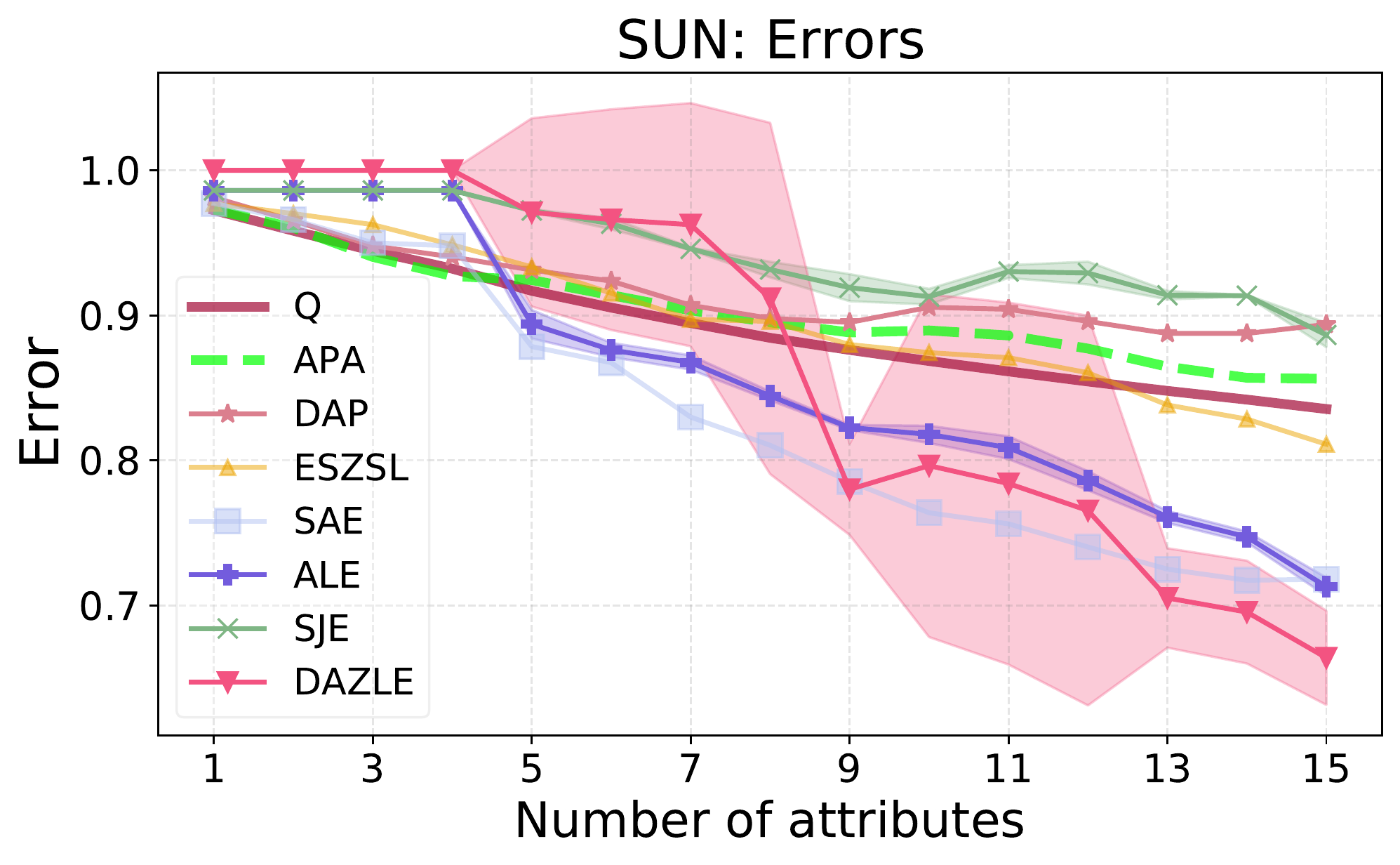}
    \caption{\textbf{Comparison of the lower bound on the error with the ZSL models error on the validation classes.} We plot the lower bound on the error (\textbf{Q}) and compare it to the error rates of the ZSL adversarial algorithm (\textbf{APA}) and the other ZSL models with attribute (\textbf{DAP, ESZSL, SAE, ALE}, and \textbf{DAZLE}). 
    The bands indicate the standard error on five runs with different seeds.}
    \label{fig:apa}
\end{figure*}

\newpage

\section{Extension to Incomplete Class-Attribute Information}

In this paper, we assume that we are provided a class-attribute matrix $A \in [0,1]^{k \times n}$ such that for each $i \in [n]$ and $j \in [k]$, the entry $A_{j,i}$ provides the probability that we observe attribute $i$ (i.e., $\psi_i(x)=1$) given that we sample an element of class $j$ (i.e., $y(x) = j$). Formally, the entries of the class-attribute matrix follow equation \eqref{matrix-class-attribute-relation}
\begin{align*}
    A_{j,i} = \Pr_{x \sim \mathcal{D}}[ \psi_i(x) = 1 | y(x) = j] \enspace .
\end{align*}
In some Zero-Shot Learning problems \citep{jayaraman2014zero,wang2017zero}, we are provided a  \textit{incomplete} class-attribute matrix. That is, for each class, we are provided reliable information only for a subset of the $n$ attributes. Formally, we are given a matrix $A \in ( [0,1]\cup \{*\} )^{k \times n} $, where the symbol $`*'$ denotes a lack of information. That is, if $A_{j,i} = *$, then the probability of observing attribute $i$ given a sample of an element of class $j$ is arbitrary. Conversely, if $A_{j,i} \in [0,1]$, then we are provided the same information considered in the original setting, and we know that the relation between attribute $i$ and class $j$ follows equation \eqref{matrix-class-attribute-relation}.

We can easily extend the lower bound developed in \Cref{sec:adv} for incomplete class-attribute matrices. In fact, in the original formulation in the paper, each entry of the class-attribute matrix defines a constraint on the joint distribution $p$ over classes and attributes. If we are not given a relation between class $j$ and attribute $i$, i.e. $A_{j,i} = *$, then we simply do not specify that constraint. The lower bound can still be computed by using a Linear Program as in \Cref{sec:compute_bound}. In the case of incomplete class-attribute matrix, we specify the constraints $(a)$ of the Linear Program only for pairs of attribute $i$ and class $j$ such that $A_{j,i} \in [0,1]$.

\end{document}